\documentclass[a4paper, 12pt]{article}

\usepackage[letterpaper, top=2.5cm, bottom=2.5cm, left=2.0cm, right=2.0cm, marginparwidth=1cm]{geometry}

\usepackage{graphicx} 
\usepackage{upgreek}
\usepackage{amsmath}
\usepackage{amsthm}
\usepackage{amsfonts}
\usepackage{amssymb}
\usepackage{mathtools}
\usepackage{bm}
\usepackage{mathrsfs}
\usepackage{amssymb}
\usepackage{caption}
\usepackage{fancyhdr}
\usepackage{booktabs}
\usepackage{float}
\usepackage{diagbox}
\usepackage{listings, comment}
\usepackage{natbib}
\usepackage{enumitem}
\usepackage{mdframed}
\usepackage{authblk}
\RequirePackage[colorlinks,citecolor=blue,urlcolor=blue]{hyperref}

\newtheorem{theorem}{Theorem}
\newtheorem{proposition}{Proposition}
\newtheorem{lemma}{Lemma}
\newtheorem{assumption}{Assumption}

\newtheorem{remark}{Remark}

\usepackage{algorithm}%
\usepackage{algorithmicx}%
\usepackage{algpseudocode}%

\usepackage{setspace}

\usepackage{tikz}
\usetikzlibrary{patterns}
\usetikzlibrary{decorations.pathreplacing,calligraphy}
\usetikzlibrary{patterns.meta}
\usetikzlibrary{arrows.meta, positioning, shapes.multipart, calc, fit}
\usetikzlibrary{shapes.geometric}
\usetikzlibrary{shadows}
\usetikzlibrary{intersections,calc}
\tikzdeclarepattern{
  name=mylines,
  parameters={
      \pgfkeysvalueof{/pgf/pattern keys/size},
      \pgfkeysvalueof{/pgf/pattern keys/angle},
      \pgfkeysvalueof{/pgf/pattern keys/line width},
  },
  bounding box={
    (0,-0.5*\pgfkeysvalueof{/pgf/pattern keys/line width}) and
    (\pgfkeysvalueof{/pgf/pattern keys/size},
0.5*\pgfkeysvalueof{/pgf/pattern keys/line width})},
  tile size={(\pgfkeysvalueof{/pgf/pattern keys/size},
\pgfkeysvalueof{/pgf/pattern keys/size})},
  tile transformation={rotate=\pgfkeysvalueof{/pgf/pattern keys/angle}},
  defaults={
    size/.initial=5pt,
    angle/.initial=45,
    line width/.initial=.4pt,
  },
  code={
      \draw [line width=\pgfkeysvalueof{/pgf/pattern keys/line width}]
        (0,0) -- (\pgfkeysvalueof{/pgf/pattern keys/size},0);
  },
}
\tikzset{global scale/.style={
    scale=#1,
    every node/.append style={scale=#1}
  }
}

\renewcommand{\hat}[1]{\widehat{#1}}
\usepackage[english]{babel}
\newcommand{\mI}{\mathcal{I}}

\newcommand{\mS}{\mathcal{S}} 
\newcommand{\xj}[1]{ X_{\mathcal{I}_{#1}} }

\newif\ifEnon
\Enontrue
\ifEnon
\newcommand{\zhifan}[1]{\color{red}{[{\bf Zhifan:} #1}]\color{black}\,}
\else
\newcommand{\zhifan}[1]{}
\fi

\begin{document}

\title{High-dimensional online learning via asynchronous decomposition: Non-divergent results, dynamic regularization, and beyond}

\author[1]{Shixiang Liu\thanks{The author thanks Chenghao Zhou from the Chinese Academy of Agricultural Sciences for a warm discussion.}} 
\author[2]{Zhifan Li}
\author[3]{Hanming Yang}
\author[4, 1]{Jianxin Yin}

\affil[1]{\footnotesize School of Statistics, Renmin University of China}
\affil[2]{\footnotesize School of Statistics and Mathematics, Zhongnan University of Economics and Law}
\affil[3]{\footnotesize Institute of Statistics and Big Data, Renmin University of China}
\affil[4]{\footnotesize Center for Applied Statistics, Renmin University of China}

\date{}
\maketitle \sloppy

\begin{abstract}
Existing high-dimensional online learning methods often face the challenge that their error bounds, or per-batch sample sizes, diverge as the number of data batches increases. 
To address this issue, we propose an asynchronous decomposition framework that leverages summary statistics to construct a surrogate score function for current-batch learning. 
This framework is implemented via a dynamic-regularized iterative hard thresholding algorithm, providing a computationally and memory-efficient solution for sparse online optimization. 
We provide a unified theoretical analysis that accounts for both the streaming computational error and statistical accuracy, establishing that our estimator maintains non-divergent error bounds and $\ell_0$ sparsity across all batches.
Furthermore, the proposed estimator adaptively achieves additional gains as batches accumulate, attaining the oracle accuracy as if the entire historical dataset were accessible and the true support were known. 
These theoretical properties are further illustrated through an example of the generalized linear model.

\end{abstract}

\section{Introduction}\label{sec: intro}
In the era of big data, streaming or online data have become increasingly prevalent in fields such as high-frequency financial trading and large-scale network analysis.
Because streaming data are continuously generated at high volume and velocity, it is often computationally impractical to store all historical observations for offline analysis.
Consequently, statistical learning should be conducted using incremental algorithms that update estimators sequentially.
Consider the linear model as an illustrative example, where data batches $D_j = \left( X_{\mI_j}, Y_{\mI_j} \right)\in \mathbb R^{n_j \times p} \times \mathbb R^{n_j}$ arrive at each time $j \ge 1$:
\begin{equation}\label{eq: example}
    Y_{\mI_j} = X_{\mI_j} \beta^* + \xi_{\mI_j} \in \mathbb R^{n_j},
\end{equation}
Here $n_j$ is the batch sample size and $\mathcal{I}_j$ denotes the observation index set for the $j$-th batch.
In this online setting, the goal is to update the preceding estimator $\hat{\beta}^{(j-1)}$ by leveraging the current batch $\mathcal{D}_j$ and a set of cumulative summary statistics, thereby avoiding the need to re-access historical raw data.
In this study, we focus on the high-dimensional streaming problem 
and propose an asynchronous decomposition framework for sparse online learning.

\subsection{Related work and marginal contribution}
\paragraph{Upper bounds in High-dimensional online learning}
In the online setting, stochastic gradient descent (SGD) and Polyak-Ruppert averaging procedure (Averaging SGD, \citet{Ruppert1988ASGD, Polyak1992ASGD}) are widely used for their $O(p)$ storage efficiency and effectiveness in low-dimensional inference \citep{zhu2022beyond, Lee2025quantile}.
However, in high-dimensional cases, these methods are known to be sub-optimal, yielding a $\ell_2$ estimation rate of order $N^{-1/4}$ \citep{Agarwal2012lower}.

Building on Nesterov’s dual averaging framework \citep{Nesterov2009primal}, \citet{Agarwal2012RADAR} introduced the Regularization Annealed epoch Dual AveRaging (RADAR) algorithm, 
which attain minimax near-optimality in the high-dimensional case. 
This framework has been extended to online debiased inference for high-dimensional linear models \citep{Chen2020SGD, Luo2023inference} and Generalized Linear Models (GLMs) \citep{Han2025adaptivedebiased}. 
However, a limitation of RADAR is its reliance on \textit{geometrically increasing batch sizes} (e.g., $n_b \asymp 2^{b-1} n_1$), which might be difficult to achieve in practice. Although some variants with fixed epoch lengths have been explored (e.g., Section 3.3 in \citet{Agarwal2012RADAR}, \citet{Juditsky2023asy}), these methods often depend on knowing the total sample size $N$ (a preset budget) in advance and are thus not directly applicable to a fully online setting.

Within a batch-wise framework, \citet{Song2020renewable} proposed a renewable online learning method, where the maximum likelihood estimator is recursively updated using the current data batch and historical summary statistics, which requires $O(p^2)$ storage. 
\citet{luo2023GLM} extended this framework to high-dimensional GLMs by incorporating $\ell_1$ regularization and developing an online debiasing procedure. 
It has also been generalized to high-dimensional single-index models \citep{Huang2024SIM}, support vector machines \citep{rao2025svm, zhang2025svm}, and quantile regression \citep{Kong2025SQR}. 
Despite its versatility, the high-dimensional renewable framework suffers from the \textit{exponential inflation of its theoretical error bounds}, which grow as $\| \hat\beta^{(b)} - \beta^*\|_2 \le  C^b \sqrt{(s\log p)/N_b}$ \citep{luo2023GLM, Huang2024SIM, rao2025svm, Kong2025SQR}, where $b\ge 1$ is the batch index, $N_b = \sum_{j=1}^b n_j$ is the cumulative sample size, and $C >1$ is a universal constant. 
This implies that, unless the cumulative sample size $N_b$ scales \textit{exponentially} with $b$, the error bound and subsequent asymptotic inference become statistically uninformative as $b$ grows large. 

Building on the aforementioned analysis, we pose the primary research question that drives this study: 
\begin{mdframed}[
  linewidth=0.5pt,
  innerleftmargin=8pt,
  innerrightmargin=8pt
]
\itshape
\textbf{Question 1:} Can we develop a high-dimensional online learning framework that operates with non-divergent batch sizes and maintains non-divergent error bounds?
\end{mdframed}

\paragraph{Influence of signal strength}
In a high-dimensional sparse model with $n$ observations, the signal strength of non-zero components plays a pivotal role in both signal estimation and inference. 
From an information-theoretic perspective, the necessity of signal strength in a variable selection task is well studied \citep{WJM07rec, butucea18AOS, castillo24sharp}. 
Furthermore, by utilizing the interplay between estimation and support recovery, \citet{Ndaoud2019interpaly} systematically quantified the influence of signal strength via a minimax phase-transition lower bound: when $\min_{i: \beta^*_i \ne 0} |\beta_i^*| \gtrsim \sqrt{{\log(ep/s)}/n}$, the $\ell_2$ estimation error $\| \hat \beta - \beta^*\|_2$ improves from the standard minimax rate $\sqrt{{s \log (ep/s)}/{n}}$ to the oracle $\ell_2$ rate $\sqrt{{s}/{n}}$ (as if the support were known). 
This phenomenon motivates the development of signal-adaptive procedures \citep{Ndaoud2020, Zhou22quantile} that preserve minimax optimality when signals are weak, and adaptively attain the oracle rate when signals are relatively strong, which is a property that standard Lasso-type estimators do not generally provide \citep{lounici11lower, bellec2018noise}.

In the batch-wise online setting, the minimum signal condition (required for oracle efficiency) relaxes as the cumulative sample size $N_b$ grows.
Consequently, for any fixed $\beta^*$, this condition is inevitably satisfied once the batch index $b$ is sufficiently large.
This suggests that online learning may benefit from two mechanisms: (i) reducing estimation error by increasing sample size, and (ii) obtaining further gains 
once the minimum signal condition is met, after which the $\ell_2$ error may improve to the oracle rate $\sqrt{s/N_b}$.
Existing online learning literature rarely treats these two phenomena jointly, and thus, we pose the second question: 

\begin{mdframed}[
  linewidth=0.5pt,
  innerleftmargin=8pt,
  innerrightmargin=8pt
]
\itshape
\textbf{Question 2:} Can we develop a high-dimensional online learning framework that benefits from both ``increasing sample size'' and ``evolving minimum signal condition''?
\end{mdframed}

\paragraph{Optimization error and statistical error} 
To implement high-dimensional renewable estimation, one typically constructs lasso-based M-estimators sequentially on each incoming batch \citep{luo2023GLM, Huang2024SIM, Kong2025SQR}.
Although the $\ell_1$-regularized optimization is a convex program, optimization errors are still unavoidable in practice due to numerical precision and algorithmic tolerances \citep{Bottou2007tradeoff, Agarwal2012global}.
Within the renewable framework, these optimization errors may accumulate over batches, potentially degrading the estimation accuracy of subsequent batches.
However, the impact of batch-wise optimization errors remains unexplored.
This observation motivates the following question:

\begin{mdframed}[
  linewidth=0.5pt,
  innerleftmargin=8pt,
  innerrightmargin=8pt
]
\itshape
\textbf{Question 3:} Can we simultaneously and sequentially control both the optimization error and the statistical accuracy in every batch?
\end{mdframed}

\paragraph{Main contributions}
We provide affirmative answers to the three aforementioned questions:
In Section \ref{sec: method}, we introduce an asynchronous decomposition framework for high-dimensional online learning with streaming data, and implement it with a dynamic iterative hard thresholding algorithm (AD-IHT).
This approach enables both computational and memory efficiency for per-batch updates. 
In Section \ref{sec: theory}, we provide general theoretical foundations showing that our estimation error remains non-divergent uniformly, and the required per-batch sample size does not grow exponentially with the number of batches (\textbf{Question 1}).
From the perspective of signal strength, we further show that our estimator adaptively sharpens as batches accumulate, ultimately achieving the oracle rate as if one had the entire dataset and knew the true support (\textbf{Question 2}).
These properties are then illustrated through a GLM example, with the key results visualized in Figure \ref{fig:rate-descend}.
Moreover, because our analysis tracks the actual iterative procedure, it simultaneously accounts for per-batch computational error and statistical accuracy, thereby enhancing the practical relevance of our theoretical results (\textbf{Question 3}).
 
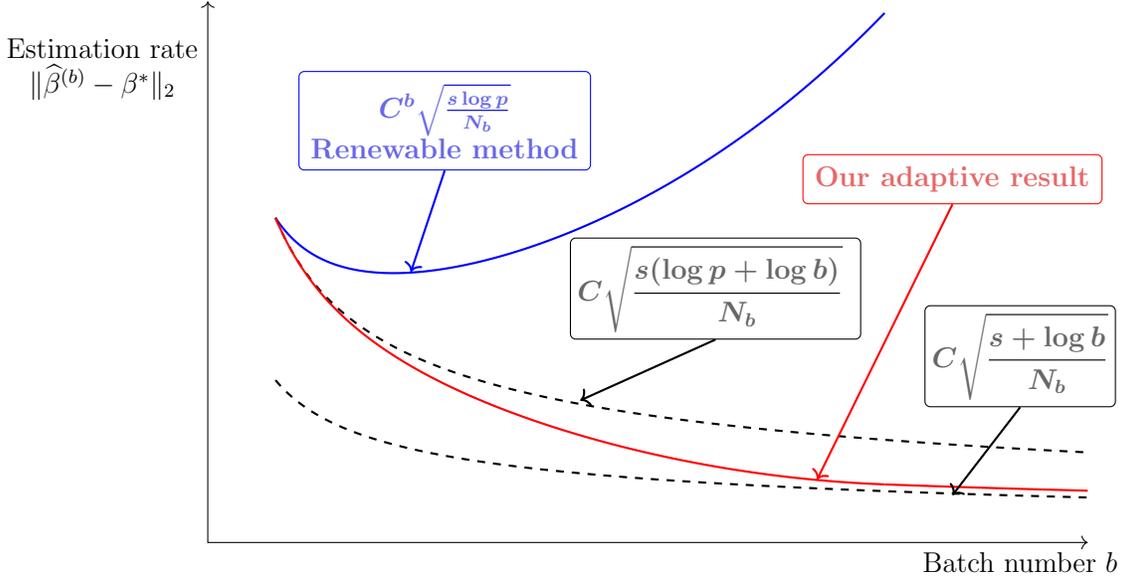
\begin{figure*}[ht]
  \centering
  \begin{tikzpicture}[global scale=0.9]
  
    \draw[<->](0,8)--(0,0)--(13,0);
    \node[left] at (0,7) {\shortstack{Estimation rate\\
    $\| \hat \beta^{(b)} - \beta^*\|_2$} };
    \node[below] at (12,0) { \shortstack{Batch number $b$} };

    \draw[blue, thick, domain=1:10, samples=200, smooth]
      plot(\x, {1.2^\x * 4/sqrt(\x)});

    \draw[black, thick, dashed, domain=1:13, samples=300, smooth]
      plot(\x, {(4.8 /sqrt(\x) });

\draw[black, thick, dashed, domain=1:13, samples=300, smooth]
      plot(\x, {2.4 /sqrt(\x) });
      
\draw[red, thick, domain=1:10, samples=300, smooth]
plot(\x, {2.4 /sqrt(\x)+0.1+ ( 1-3*( (\x-1)/9 )^2 + 2*( (\x-1)/9 )^3 ) * (2.4 /sqrt(\x)- 0.1)   } ) ;

\draw[red, thick, domain=10:13, samples=300, smooth]
  plot(\x, {2.4 /sqrt(\x)+0.1} ) ;

    \draw[->, blue, thick] (3.5, 5.5) -- (3,4);
    \node[above, blue, draw=blue, fill=white, fill opacity=0.6, align=center,  rounded corners=2pt, inner sep=5pt, text=blue!85!black,font=\bfseries\boldmath] at (3.5,5.5)
    {  $C^b \sqrt{\frac{s \log p}{N_b}}$ \\ 
    Renewable method
}; 
 
    \draw[->, black, thick] (7.5,3) -- (5.5,2.1);
    \node[above, black, draw= black, fill=white, fill opacity=0.6, rounded corners=2pt, inner sep=3pt, align=center,text=black,font=\bfseries\boldmath] at (7.5,3)
      { $\displaystyle C \sqrt{\frac{s( \log p+ \log b)}{N_b}}$ };
    
    \draw[->, black, thick] (12,2) -- (11,0.7);
     \node[above, black, draw=black, fill= white, fill opacity=0.6, rounded corners=2pt, inner sep=3pt,text=black,font=\bfseries\boldmath] at (12,2)
          {$\displaystyle C \sqrt{\frac{s + \log b}{N_b}}$};

     \draw[->,red, thick] (11,5) -- (9,0.93 );
     \node[above, red, draw=red, fill= white, fill opacity=0.6, rounded corners=2pt, inner sep=5pt,text=red!85!black ,font=\bfseries\boldmath] at (11,5)
          {Our adaptive result};

  \end{tikzpicture}
  \caption{\label{fig:rate-descend}
  Relationship between the $\ell_2$ error and the batch number $b$ in online learning of high-dimensional GLMs, where for ease of display we assume each batch contains $n$ samples, so that $N_b = \sum_{j=1}^b n_j = nb$.
  The $\log b$ term is introduced to ensure the error bounds hold uniformly for all batches $b \ge 1$.}
\end{figure*}

\subsection{Notation} 
For sequences $a_n$ and $b_n$, we write $a_n = O(b_n)$ (or $a_n \lesssim b_n$) if $a_n \le Cb_n$ for all large $n$ with some universal constant $C>0$, and $a_n \prec b_n$ if $a_n/b_n \to 0$ as $n \to \infty$.
We write that $a_n \asymp b_n$ if $a_n = O(b_n)$ and $b_n  = O(a_n)$.
Let $[m] = \{1,2,\dots,m\}$, and $\mathbf{1} (\cdot)$ be the indicator function. 
Let $\mathbf{0}_d$ denote the $d$-dimensional zero vector.
Define $x \vee y = \max\{x,y\}$.  
For sets $A$ and $B$ with sizes $|A|$ and $|B|$, let $\beta_A = (\beta_j)_{j \in A} \in \mathbb{R}^{|A|}$, and $X_{A,B} \in \mathbb{R}^{|A| \times |B|}$ be the submatrix of $X \in \mathbb R^{ n \times p}$ with rows and columns in $A$ and $B$.  
For a vector $\beta$, denote $\|\beta\|_2$ as its Euclidean norm, $\|\beta\|_0$ as the number of its nonzero entries, and $\operatorname{supp}(\beta)$ as its support set. 
For a matrix $X$, let $\| X \|_2$ denote its spectral norm. For a square matrix $X$, $\Lambda_{\max}(X)$ and $\Lambda_{\min}(X)$ denote its maximum and minimum eigenvalues , respectively.
For any $\mu \in \mathbb{R}^p$ and $\gamma > 0$, let $\mathbb{B}_q^p(\mu, \gamma) = \{x \in \mathbb{R}^p: \|x - \mu\|_q \le \gamma\}$ denote the $\ell_q$-ball in $\mathbb{R}^p$ centered at $\mu$ with radius $\gamma$. For simplicity, we write $\mathbb{B}_q^p(\gamma)$ when $\mu = \mathbf{0}_p$. 
Let $c, C, C_1, C_e,\ldots$ denote positive and universal constants whose actual values may vary from line to line.

\section{Methodology}\label{sec: method}

\subsection{Preliminary}
Here, we give a general setup for high-dimensional online learning.
Assume data arrive sequentially in batches. Let $D_j$ denote the $j$-th data batch with sample size $n_j$, and let $N_j = \sum_{k=1}^j n_k$ be the cumulative sample size up to batch $j$. 
Each $D_j$ has an associated twice-differentiable loss $f_j(D_j;\beta)$, abbreviated $f_j(\beta)$, where $\beta\in\mathbb{R}^p$ is the high-dimensional parameter of interest and we aim to learn sparse solutions sequentially. 

For example, in the linear model \eqref{eq: example}, the data for the \(j\)-th batch are \(D_j = (X_{\mI_j}, Y_{\mI_j})\in \mathbb R^{n_j \times p} \times \mathbb R^{n_j}\), where $\mI_j = \{N_{j-1}+1, \cdots,  N_j \}$ is the observation index for this batch. The ordinary least squares loss for the $j$-th batch is \(f_j(\beta) = \| Y_{\mI_j} - X_{\mI_j} \beta\|_2^2\).

\subsection{Revisit high-dimensional renewable learning}
This subsection introduces the well-known renewable method \citep{Song2020renewable} from a score function perspective and discusses the main issues that arise in the high-dimensional setting.

\paragraph{Renewable learning}
Suppose that an initial estimator $\hat{\beta}^{(1)}$ has already been obtained from the first batch data $D_1$. When the second batch \(D_2\) arrives, \citet{Song2020renewable} approximate the cumulative score $\nabla_\beta f_2(\beta) + \nabla_\beta f_1(\beta)$ via a Taylor expansion around $\hat{\beta}^{(1)}$:
\begin{equation}\label{eq: renew ini score decomposition}
\begin{aligned}
\nabla_\beta f_2( \beta ) + \nabla_\beta f_1( \beta)  
= &  \nabla_\beta f_2( \beta ) +  \nabla_\beta f_1(\hat \beta^{(1)})  + \nabla_\beta^2 f_1( \hat \beta^{(1)} )  (\beta  - \hat \beta^{(1)} ) + N_1\mathcal O_p( \| \hat\beta^{(1)} - \beta \|_2^2 )\\ 
\approx  & \nabla_\beta f_2(  \beta ) + \nabla_\beta^2 f_1( \hat \beta^{(1)} )  (\beta  - \hat \beta^{(1)} ) ,
\end{aligned}
\end{equation}
where the last line ignores higher-order error and approximates \(\nabla_\beta f_1(\hat \beta^{(1)}) \approx 0_p\). 
Consequently, only the summary Hessian $\nabla_\beta^2 f_1(\hat \beta^{(1)})$ needs to be retained from $D_1$ to construct the approximate score function.
We denote the resulting estimator based on this surrogate score as $\hat{\beta}^{(2)}$. 
Similarly, to get the approximate cumulative score for any batch $b \ge 2$, this approach generalizes by expanding the historical score $\sum_{j=1}^{b-1} \nabla_\beta f_j(\beta)$ around the previous estimate $\hat{\beta}^{(b-1)}$:
\begin{equation}\label{eq: renewable batch b}
\begin{aligned}
& \nabla_\beta f_b(  \beta ) + \sum_{j=1}^{b-1} \nabla_\beta f_j( \beta ) \\
= &  \nabla_\beta f_b( \beta ) + \sum_{j=1}^{b-1} \nabla_\beta f_j( \hat \beta^{(b-1)} ) + \sum_{j=1}^{b-1} \nabla_\beta^2 f_j( \hat \beta^{(b-1)} )~ (\beta  - \hat \beta^{(b-1)}) + N_{b-1} \mathcal O_p( \| \beta  - \hat \beta^{(b-1)}\|_2^2) \\ 
\approx  & \nabla_\beta f_b( \beta ) + \sum_{j=1}^{b-1} \nabla_\beta^2 f_j( \hat \beta^{(j)} ) ~(\beta  - \hat \beta^{(b-1)} ) ,
\end{aligned}
\end{equation}
where the last step (i) approximates \( \sum_{j=1}^{b-1} \nabla_\beta f_j( \hat \beta^{(b-1)} ) \approx 0_p\), and (ii) replaces the Hessian matrix \( \sum_{j=1}^{b-1} \nabla_\beta^2 f_j( \hat \beta^{(b-1)} )\) with \( \sum_{j=1}^{b-1} \nabla_\beta^2 f_j( \hat \beta^{(j)} )\).
For high-dimensional sparse settings, \citet{luo2023GLM} further propose to get the regularized estimator via optimizing an $\ell_1$-penalized surrogate loss:
\begin{equation}\label{eq: renewable}
\hat \beta^{(b)} \in \arg\min_{\beta \in \mathbb R^p} \left\{ \frac{1}{N_b} \left( f_b(  \beta )  +  \frac12 (\beta  - \hat \beta^{(b-1)} )^\top ~ \sum_{j=1}^{b-1} \nabla_\beta^2 f_j( \hat \beta^{(j)} )~ (\beta  - \hat \beta^{(b-1)} ) \right) + \lambda_b \| \beta\|_1 \right\}.
\end{equation}
Therefore, during the online learning process, only the cumulative Hessian \( \sum_{j=1}^{b-1} \nabla_\beta^2 f_j( \hat \beta^{(j)} )\) needs to be updated, requiring \(O(p^2)\) storage. 
In essence, the (penalized) renewable method constructs a surrogate score via first-order expansion and Hessian approximation. This enables effective utilization of historical information without re-accessing raw data.

\paragraph{Limitation}
The renewable method works well in low dimensions \citep{Luo2023inference, Ding2024renewable, Hu2025scollaborative}, however, it faces some challenges in high-dimensional settings. 
For instance, in model \eqref{eq: example}, under a restricted isometry property of $X_{\mI_1}$ and a $\sigma$-subGaussian property on $\xi_{\mI_1}$, on the support $\mathcal S^* = \text{supp}(\beta^*)$, a minimax (near) optimal estimator $\hat\beta^{(1)}$ with high probability satisfies \citep{Zhang2018HTP, Ndaoud2020}:
\begin{align*}
 \left\| \left( \frac{\nabla_\beta f_1(  \hat \beta^{(1)})}{N_1} \right)_{\mathcal S^*}    \right\|_2 
 \lesssim \left\|  \left(\frac{\xj{1}^\top \xj{1}}{N_1}  (\beta^*  - \hat \beta^{(1)} ) \right)_{\mathcal S^*}    \right\|_2 + \left\|  \left(\frac{\xj{1}^\top \xi_{\mathcal I_1}}{N_1}  \right)_{\mathcal S^*}   \right\|_2
\lesssim \| \hat\beta^{(1)} - \beta^* \|_2 + \sigma \sqrt{ \frac{s\log p}{N_1} }.
\end{align*}
Therefore, on set $\mathcal S^*$, the Taylor expansion in \eqref{eq: renew ini score decomposition} becomes: 
\begin{align*}
& \left( \nabla_\beta f_2( \beta ) + \nabla_\beta f_1( \beta) \right)_{\mathcal S^*} \\
= &  \left( \nabla_\beta f_2( \beta ) +  \nabla_\beta f_1(\hat \beta^{(1)})  + \nabla_\beta^2 f_1( \hat \beta^{(1)} )  (\beta  - \hat \beta^{(1)} )\right)_{\mathcal S^*} + N_1\mathcal O_p( \| \hat\beta^{(1)} - \beta \|_2^2 )\\ 
= & \left( \nabla_\beta f_2(  \beta ) + \nabla_\beta^2 f_1( \hat \beta^{(1)} )  (\beta  - \hat \beta^{(1)} )\right)_{\mathcal S^*} + N_1 \mathcal O_p\left(\| \hat\beta^{(1)} - \beta^* \|_2 +\| \hat\beta^{(1)} - \beta \|_2^2 +\sigma \sqrt{ \frac{s\log p}{N_1} } \right).
\end{align*}
Similarly, for the general $b$-th batch:
\begin{equation}\label{eq: renewable limit}
\begin{aligned}
& \left(\nabla_\beta f_b(  \beta ) + \sum_{j=1}^{b-1} \nabla_\beta f_j( \beta ) \right)_{\mathcal S^*}\\
=&  \left(\nabla_\beta f_b( \beta ) + \sum_{j=1}^{b-1} \nabla_\beta^2 f_j( \hat \beta^{(j)} ) ~(\beta  - \hat \beta^{(b-1)} )\right)_{\mathcal S^*}
+ N_{b-1} \mathcal O_p\left(\| \hat\beta^{(b-1)} - \beta^* \|_2 +\| \hat\beta^{(b-1)} - \beta \|_2^2 + \sigma \sqrt{ \frac{s\log p}{N_{b-1}} } \right) .
\end{aligned}
\end{equation}
This result indicates that the approximation error of the surrogate score is (at least) of the first order $\mathcal O_p( \| \hat\beta^{(b-1)} - \beta^* \|_2)$, instead of consisting only of the desired second-order term as in \eqref{eq: renewable batch b}.
These first-order errors accumulate across batches and potentially cause theoretical error bounds to grow with $b$. For instance, one may obtain $\|\hat\beta^{(b)}-\beta^*\|_2^2\le C^b (s\log p)/N_b$ for some constant $C>1$ \citep{luo2023GLM, Huang2024SIM, Kong2025SQR}, which limits non-divergent guarantees to relatively few batches, i.e., $b = o(\log N_b)$.

\subsection{Asynchronous decomposition framework}
To address the limitations mentioned above, we propose an asynchronous decomposition framework: 
In the $b$-th batch learning, instead of uniformly expanding every past score around the single estimate $\hat \beta^{(b-1)}$ as in \eqref{eq: renewable batch b}, we expand each past $\nabla_\beta f_j$ at its own estimate $\hat \beta^{(j)}$ respectively:
\begin{equation}\label{eq: asynchronous}
\begin{aligned}
& \nabla_\beta f_b(\beta) + \sum_{j=1}^{b-1}\nabla_\beta f_j(\beta) \\ 
=& \underbrace{ \nabla_\beta f_b(\beta) }_{\text{current batch}} +
\underbrace{ \sum_{j=1}^{b-1} \left\{ \nabla_\beta f_j(\hat \beta^{(j)})  + \nabla_\beta^2 f_j( \hat \beta^{(j)}) ~(\beta - \hat \beta^{(j)}) + \mathcal O_p( n_j 
 \| \beta - \hat \beta^{(j)} \|_2^2)\right\} }_{\text{Asynchronous decomposition}}\\
\approx& \nabla_\beta f_b(\beta) + \sum_{j=1}^{b-1} \nabla_\beta f_j(\hat \beta^{(j)}) + \left\{\sum_{j=1}^{b-1} \nabla_\beta^2 f_j( \hat \beta^{(j)})\right\} ~\beta - \sum_{j=1}^{b-1}\left\{ \nabla_\beta^2 f_j( \hat \beta^{(j)})~\hat \beta^{(j)} \right\}.
\end{aligned}
\end{equation}
This asynchronous decomposition offers two concrete improvements: (i) it retains the historical gradient $\sum_{j=1}^{b-1}\nabla_\beta f_j(\hat\beta^{(j)})$ instead of approximating it by zero, thereby eliminating the first-order approximation error presented in \eqref{eq: renewable limit}; and (ii) it avoids the additional Hessian approximation used in \eqref{eq: renewable batch b}.
Consequently, the approximation error in \eqref{eq: asynchronous} is reduced to a (weighted) second-order term $N_{b-1} \mathcal{O}_p(\sum_{j=1}^{b-1} n_j \|\hat\beta^{(j)} - \beta^*\|_2^2 / N_{b-1})$, which is smaller than the renewable expansion in \eqref{eq: renewable limit}. 
These differences underlie the enhanced stability of our framework. 
Furthermore, this decomposition leads to the following surrogate loss function for batch $b$:  
\begin{equation}\label{eq: surrogate loss}
\tilde L_b(\beta) := f_b(\beta) 
    + \sum_{j=1}^{b-1} 
    \beta^\top \left\{ \nabla_\beta f_j(\hat \beta^{(j)}) - \nabla_\beta^2 f_j( \hat \beta^{(j)})~\hat \beta^{(j)} \right\} 
    + \frac12 \sum_{j=1}^{b-1} \beta^\top \nabla_\beta^2 f_j(\hat \beta^{(j)}) ~\beta.
\end{equation}

\paragraph{Storage} 
Based on the decomposition in \eqref{eq: asynchronous} and \eqref{eq: surrogate loss}, it suffices to maintain and update the following two summary statistics: the cumulative Hessian \(\sum_{j=1}^{b-1} \nabla_\beta^2 f_j( \hat \beta^{(j)})\in \mathbb R^{p \times p}\), and the cumulative vector \(\sum_{j=1}^{b-1}\left\{\nabla_\beta f_j(\hat \beta^{(j)})- \nabla_\beta^2 f_j( \hat \beta^{(j)})~\hat \beta^{(j)} \right\} \in \mathbb R^p\). The total storage cost is \(O(p^2)\), maintaining the same order of efficiency as the renewable method. 

\begin{remark}
To establish asymptotic properties, several studies \citep{luo2023GLM, Kong2025SQR} have proposed using terms similar to those in \eqref{eq: asynchronous} to debias the renewable estimator \eqref{eq: renewable}. 
However, because the error bound of the initial renewable estimator diverges with $b$ (as we mentioned earlier), such debiasing procedures are theoretically guaranteed in the case $b=o(\log N_b)$.
In contrast, in Section \ref{sec: theory} we show that the asynchronous decomposition yields non-divergent error bounds for every $b\ge1$, which is the main advantage of the proposed framework. 
\end{remark}

\subsection{Algorithm implementation}
Building on the surrogate gradient \eqref{eq: asynchronous} and loss function \eqref{eq: surrogate loss}, we present an Iterative Hard Thresholding (IHT, \citet{BLUMENSATH2009265})-based procedure to obtain sparse online estimates.
Rather than employing a conventional M-estimator based on empirical risk minimization, we adopt this algorithm-based regularization approach primarily because it allows for a unified analysis of both statistical accuracy and computational (optimization) error incurred by iterative updates.  
This prevents computational error from accumulating uncontrollably across batches.

We first define the entrywise hard thresholding operator $\mathcal{T}^p_{\lambda}:\mathbb R^p\to\mathbb R^p$: 
\begin{equation*} 
\begin{aligned}
    \Big(\mathcal{T}^p_{\lambda}(z)\Big)_j &:= z_j \times \mathbf{1}\left( |z_j|\geq \lambda \right), \quad\text{for every } z\in\mathbb R^p \text{ and } j \in [p].
\end{aligned}
\end{equation*}
Let \(\hat \beta^{(b,t)}\) denote the parameter after the \(t\)-th iteration within the \(b\)-th batch. For each batch \(b \ge 1\), we initialize with a cold start, i.e., \(\hat \beta^{(b,0)} = \mathbf{0}_p\). The iterative procedure for batch \(b\) proceeds as follows:

\noindent\textbf{\mbox{Step 1:}} For each $t\ge0$, perform a gradient descent step based on \eqref{eq: asynchronous} to obtain an intermediate state $H^{(b,t+1)}$. 
Decay the threshold $\lambda_\beta^{(b,t+1)}$ from the previous iteration. 
The specific learning rate $\eta_b>0$ and the decay rate $\kappa\in(0,1)$ will be detailed in Section \ref{sec: theory}.

\noindent\textbf{\mbox{Step 2:}} Apply the entrywise hard thresholding operator $\mathcal{T}^p_{\lambda_\beta^{(b,t+1)}}$ to $H^{(b,t+1)}$ and obtain the updated sparse estimate $\hat{\beta}^{(b,t+1)}$. 

\noindent\textbf{\mbox{Step 3:}} Iterate Steps 1-2 until the threshold decays to a pre-specified limit $\lambda_\beta^{(b,\infty)}$. Then run additional $C_1 \log N_b$ refinement iterations with this fixed threshold. 
Output $\hat{\beta}^{(b)}=\hat\beta^{(b,t)}$.

\noindent\textbf{\mbox{Step 4:}} Update the cumulative summary statistics using \(\hat \beta^{(b)}\) and the data $D_b$. Then release data $D_b$ from the memory.

The procedure is detailed in Algorithm \ref{alg: asyn}, which uses the hard thresholding operator to enforce sparsity in the high-dimensional parameters. 
Inspired by \citet{Fan2018ILAMM, Ndaoud2020, she2023tit}, we employ a sequence of nonincreasing thresholds to control the number of selected variables at each iteration, thereby achieving an explicit trade-off between computational efficiency and statistical accuracy. 
Once the threshold decays to the prespecified floor $\lambda_\beta^{(b, \infty)}$, it remains fixed, and the algorithm runs additional $C_1 \log N$ iterations to fully refine the estimate. 
Notably, the initialization for each batch can also be set to the final output of the preceding batch, yielding a warm start $\hat{\beta}^{(b,0)} = \hat{\beta}^{(b-1)}$.
Both warm and cold starts share the same theoretical guarantees, so either initialization is feasible in practice.
Furthermore, as hard thresholding does not impose additional shrinkage on signal magnitudes \citep{john17gauss}, our procedure enables sharper error characterizations for (relatively) strong signals (see Theorems \ref{th: sharper batch} and \ref{th: glm sharper}).

\begin{algorithm}[!htbp]
	\caption{\textsc{Asynchronous Decomposition via IHT (AD-IHT)}}\label{alg: asyn}
	\begin{algorithmic}[1]
		\setstretch{1.5}
		\Require $\kappa$ 
		\State Set $N=0,~~\hat \beta^{(0)}= \mathbf 0_p$,~~$\text{Inter}  = \mathbf 0_p$,~~
		$\text{Hess}  = \mathbf 0_{p\times p}$ 
		\For{$b = 1,2,\ldots $,}
		\item[] \hspace{0.5em} \textbf{Input:} $D_b,~~ \eta_b,~ ~\lambda_\beta^{(b,0)},~  \lambda_\beta^{(b,\infty)} $
		\State $t \Leftarrow 0$,~~ $\hat \beta^{(b,0)} \Leftarrow \mathbf{0}_p$
		\State $N \Leftarrow N + \text{sample size of } D_b$ 
		\While{$t <  \log_\kappa (\lambda_\beta^{(b,\infty)}/\lambda_\beta^{(b,0)} )+ C_1 \log N$,}
		\State $\begin{aligned}[t] 
			H^{(b,t+1)} \Leftarrow \hat\beta^{(b,t)} - \eta_b \left\{ \nabla_\beta f_b(\hat\beta^{(b,t)}) + \text{Inter}  + \text{Hess}~ \hat\beta^{(b,t)}  \right\}
		\end{aligned}$
		\State $\lambda_\beta^{(b,t+1)} \Leftarrow \max\left(\kappa \lambda_\beta^{(b,t)},~ \lambda_\beta^{(b,\infty)} \right)$
		\State $\hat{\beta}^{(b,t+1)} \Leftarrow \mathcal{T}^p_{\lambda_\beta^{(b,t+1)}}\left( H^{(b,t+1)}\right)$
		\State $t \Leftarrow t + 1$
		\EndWhile
		\State $\hat{\beta}^{(b)} \Leftarrow  \hat\beta^{(b,t+1)}$
		\item[] \hspace{0.5em} \textbf{Output:} $\hat{\beta}^{(b)}$
		\State $\text{Inter} \Leftarrow \text{Inter} + \nabla_\beta f_b(\hat{\beta}^{(b)}) - \nabla_\beta^2 f_b(\hat{\beta}^{(b)})~ \hat{\beta}^{(b)}$, ~~
		$\text{Hess} \Leftarrow \text{Hess} + \nabla_\beta^2 f_b(\hat{\beta}^{(b)})$ 
		\State Release data $D_b$ from the memory
		\EndFor
	\end{algorithmic}
\end{algorithm}

\section{Non-asymptotic theory}\label{sec: theory}
This section presents theoretical guarantees for the asynchronous decomposition framework as implemented in Algorithm \ref{alg: asyn}.
In Section \ref{subsec: general} we state the underlying assumptions and establish general results. 
In Section \ref{subsec: glm}, we illustrate the theoretical advantages through a concrete example of generalized linear models (GLMs, \citet{mccullagh1989GLM}).

\subsection{General results}\label{subsec: general}

This subsection analyzes the estimation error of Algorithm \ref{alg: asyn} relative to a target sparse parameter $\bar{\beta} \in \mathbb{R}^p$ with sparsity $s := \|\bar{\beta}\|_0$. The vector $\bar{\beta}$ may represent the ground truth parameter $\beta^*$ or its best $s$-sparse approximation.
For notational simplicity, we omit the subscript $\beta$ in the gradient and Hessian operators.
We impose the following regularity conditions on the loss functions $f_j$ for each batch $j \ge 1$: 

\begin{assumption}[Restricted Strong convexity and Smoothness]\label{assump: rip}
Each $f_j$ is twice-differentiable, and for any index set $\mathcal S \subset [p]$ satisfying $| \mathcal S| \le Cs$, the eigenvalues of the restricted Hessian are bounded such that: 
\begin{equation}\label{eq: rip}
m n_j \le \Lambda_{\min} \left\{ \left(\nabla^2 f_j(\beta) \right)_{\mathcal S, \mathcal S} \right\}\le  \Lambda_{\max}\left\{ \left(\nabla^2 f_j(\beta) \right)_{\mathcal S, \mathcal S} \right\} \le M n_j,\quad
 \text{for each } j\ge 1, ~\beta \in \mathbb B_0^p(Cs),
\end{equation}
where $m$ and $M$ are universal positive constants.
We call $f_j$ satisfies RSS$(m,M,C s)$ iff \eqref{eq: rip} holds.
\end{assumption}

\begin{assumption}[Restricted second-order smoothness]\label{assump: lipschitz}
Each gradient $\nabla f_j(\beta)$ exhibits local second-order smoothness near $\bar{\beta}$. Specifically, there exist parameters $L_j, \delta_j > 0$ such that:
\begin{equation}\label{eq: gradsmooth}
\sup_{\mathcal S \subset [p]: ~|\mathcal S| \le (C+1)s} \left\| \left\{ \nabla f_j (\bar\beta) - \nabla f_j (\gamma) -  \nabla^2 f_j (\gamma)  (\bar \beta - \gamma) \right\}_{\mathcal S} \right\|_2 \le L_j \| \bar \beta - \gamma \|_2^2,
\end{equation}
for all $\gamma \in \mathbb{B}_0^p(Cs) \bigcap \mathbb B_2^p(\bar\beta, \delta_j)$. We say that $f_j$ satisfies RGS$(L_j, \delta_j, Cs)$ iff \eqref{eq: gradsmooth} holds.
\end{assumption}

\begin{assumption}[Progressive error control]\label{assump: N1}
 There exists a sequence $\{\alpha_j\}_{j \ge 1}$ satisfying $\alpha_j \ge \| \sum_{k=1}^j \nabla f_k(\bar{\beta}) \|_{\infty}$ for all $j \ge 1$. Furthermore, for every $b \ge 1$, the following condition holds:
\begin{equation}\label{eq: N1}
C_p \sum_{j=1}^b \frac{s L_j}{N_j^2} \cdot \alpha_j^2 \le \sqrt{s} \alpha_{b+1},
\end{equation}
where $C_p > 0$ is a sufficiently large universal constant.
\end{assumption}

Among these, Assumption \ref{assump: rip} ensures the contraction mapping behavior of the iterative algorithm, guaranteeing that the estimation error decays as the number of iterations increases. Such a condition is a standard requirement in the literature on iterative optimization \citep{jain2014iterative, Zhang2018HTP, Fan2018ILAMM}.
Assumption~\ref{assump: lipschitz} controls the restricted first-order approximation error of the gradient (score function) and thus characterizes the smoothness of the Hessian. This condition is commonly employed in the streaming data literature \citep{Song2020renewable}.
Assumption \ref{assump: N1} quantifies how the cumulative errors affect the next-batch estimate. It is the key assumption that controls the error propagation across batches. 
In Section \ref{subsec: glm} (GLM with sub-Gaussian random design), we show that under mild sample-size conditions, these three assumptions simultaneously hold in every batch with a high probability, thereby justifying their validity.

We first present the theoretical results for the first two data batches to establish the inductive basis. 
\begin{proposition}[Burn-in]\label{prop: 12}  
Suppose $f_b$ satisfies RSS$(m, M, (2C_s + 1)s )$ for $b=1,2$, and $f_1$ satisfies RGS$\left( L_1, C_e \sqrt s \lambda_\beta^{(1,\infty)}, (C_s +1)s \right)$.  
Assume Assumption \ref{assump: N1} holds with $C_p = \frac{4C_e^2 C_\beta}{m+M}$.
For $b = 1, 2$, let the learning rate $\eta_b \in \left[ \frac1{(m+M)N_b}, \frac2{(m+M)N_b} \right]$, the decay rate $\kappa \in ( \frac{M}{m+M},1) $.
If the regularization parameters satisfy:
$$
\begin{aligned}
    \lambda_\beta^{(b,0)} \ge  \frac{1}{2\sqrt s}\|  \bar \beta \|_2 ,
    \quad \lambda_\beta^{(b,\infty)} = C_\beta \cdot \frac{\alpha_b}{N_b} , 
\end{aligned}
$$  
then the following $\ell_2$ and $\ell_0$ bounds hold:
\begin{equation*} 
\begin{aligned}
\| \hat \beta^{(b,t)} - \bar \beta\|_2 \le C_e \cdot \sqrt s \lambda_\beta^{(b,t)},
    \quad   \| \hat \beta^{(b,t)}  \|_0 \le  (1+C_s) \cdot s,
    \quad\text{ for } b=1,2,~ t\ge 0,
\end{aligned}
\end{equation*}
where $C_\beta := \frac{8}{m+M} \cdot \frac{m+M+M/\kappa}{m+M - M/\kappa}$, $C_e := \frac{2(m+M) }{m+M - M/\kappa}$, and $C_s := (C_e -1)^2$ are universal constants that depend solely on $m$, $M$, and $\kappa$. 
\end{proposition}

Proposition \ref{prop: 12} demonstrates that the AD-IHT outputs of the first two batches yield $\ell_0$ sparse estimators, and their $\ell_2$ errors are controlled by the (cumulative) score function evaluated at $\bar \beta$. Building upon these initial results, we then generalize the estimation properties to all subsequent batches.

\begin{theorem}[General batch]\label{th: minimax batch}
For every $b \ge 1$, suppose $f_b$ satisfies RSS$(m, M, (2C_s + 1)s )$ and RGS$\left( L_b, C_e \sqrt s \lambda_\beta^{(b,\infty)}, (C_s +1)s \right)$. 
Suppose Assumption \ref{assump: N1} holds with $C_p =  \frac{4C_e^2 C_\beta}{m+M}$.
Let the learning rate $\eta_b \in \left[ \frac1{(m+M)N_b}, \frac2{(m+M)N_b} \right]$, the decay rate $ \kappa \in ( \frac{M}{m+M},1) $.
If the regularization parameters satisfy:
$$
\begin{aligned}
    \lambda_\beta^{(b,0)} \ge  \frac{1}{2\sqrt s}\|  \bar \beta \|_2 ,
    \quad \lambda_\beta^{(b,\infty)} \ge  C_\beta \cdot \frac{\alpha_b}{ N_b},
    \quad \text{ for every }b \ge 1,
\end{aligned}
$$
then the following $\ell_2$ and $\ell_0$ bounds hold: 
\begin{equation}\label{eq: general error}
\begin{aligned}
\| \hat \beta^{(b,t)} - \bar \beta\|_2 \le C_e \cdot \sqrt s \lambda_\beta^{(b,t)},
    \quad   \| \hat \beta^{(b,t)} \|_0 \le  (1+C_s)\cdot s,
\quad\text{ for every } b\ge 1,~ t\ge 0,
\end{aligned}
\end{equation}
where $C_\beta, C_e,$ and $C_s$ are the same universal constants from Proposition \ref{prop: 12}, which are determined solely on $m$, $M$, and $\kappa$.
\end{theorem}
Theorem \ref{th: minimax batch} implies that by setting the predetermined threshold $\lambda_\beta^{(b, \infty)} \asymp \frac{\alpha_b }{N_b}$, the final estimator for each batch $b$ satisfies:
\begin{equation*} 
\| \hat{\beta}^{(b)} - \bar{\beta} \|_2 \le C \sqrt{s} \cdot \frac{\alpha_b}{N_b}.
\end{equation*}
Crucially, the constant $C$ does not diverge as $b$ increases, indicating that the estimation error remains non-divergent regardless of the number of batches.
Furthermore, \eqref{eq: general error} provides a pathwise error guarantee along the algorithmic trajectory, and therefore simultaneously controls computational (optimization) error and statistical accuracy for every batch. This provides a more practically relevant guarantee compared to the canonical M-estimator theory.

For a fixed target $\bar{\beta}$, as the sample size increases, the minimum signal strength required for variable selection becomes weaker \citep{WJM07rec}, which potentially further improves estimation accuracy \citep{Fan2018ILAMM, Ndaoud2019interpaly, Zhou22quantile}.
We next characterize how the signal strength and sample accumulation benefit the estimation accuracy under a streaming data setting. Define 
$$
\begin{aligned}
b_1^*:=& \inf\left\{ b \ge 1: ~\min_{i \in \bar S} |\bar \beta_i| \ge C_2 \cdot \frac{ \alpha_b }{N_b} \right\}, 
\end{aligned}
$$
The index $b_1^*$ is the number of batches required for the signal to be "relatively strong". 
Assume there exists a sequence $\{\theta_j\}_{j \ge 1}$ that serves as upper bounds for the cumulative score function on the set $\bar{\mathcal{S}} := \text{supp}(\bar{\beta})$, such that $\theta_j \ge \left\| \sum_{k=1}^{j} \left(\nabla f_k(\bar \beta)\right)_{ \bar{\mathcal S} } \right\|_{2}$ for all $j \ge 1$. Furthermore, Define
$$
\begin{aligned}
b_2^*:=& \inf\left\{ b \ge 2: ~C_3 \sum_{j=1}^{b-1} \frac{ s L_j}{N_j^2} \cdot \alpha_j^2 \le  \theta_b \right\}.
\end{aligned}
$$
The index $b_2^*$ is the number of batches required for the accumulated historical error to be ``relatively weak'', and both $C_2, C_3$ are constants that depend solely on $m,~M$, and $\kappa$.
Let $b^* := b_1^* \vee b_2^*$.
For sufficiently large batch index $b \ge b^*$, we obtain the following sharper result.

\begin{theorem}[Sharper bound]\label{th: sharper batch}
Under all conditions of Theorem \ref{th: minimax batch}, for every $b \ge b^*$, the estimator $\hat{\beta}^{(b)}$ satisfies the sharper $\ell_2$ error bound: 
\begin{equation*} 
\| \hat \beta^{(b)} - \bar \beta\|_2 \le  C_{sharp} \cdot \frac{\theta_b}{N_b},
\end{equation*}
where $C_{sharp}>0$ is a universal constant depending solely on $m$, $M$, and $\kappa$.
\end{theorem}
Compared with the bound in \eqref{eq: general error}, Theorem \ref{th: sharper batch} yields a refined rate for $b\ge b^*$ (since the upper bound of $\left\| \sum_{k=1}^{j} \left(\nabla f_k(\bar \beta)\right)_{ \bar{\mathcal S} } \right\|_{2}$ is usually sharper than that of $\sqrt s \| \sum_{k=1}^j \nabla f_k(\bar{\beta}) \|_{\infty}$).
This finding reflects the additional benefits of increasing sample size from the perspective of signal strength.
Algorithmically, this improvement stems from the fact that the hard thresholding operator imposes no shrinkage on strong signals, thereby enabling more accurate signal recovery and reducing the regularization bias typically found in the $\ell_1$-based methods \citep{lounici11lower, bellec2018noise}.

Overall, Theorems \ref{th: minimax batch} and \ref{th: sharper batch} reveal that the proposed procedure exhibits a signal-adaptive behavior: it achieves non-divergent error guarantees when the number of batches is small, and attains a sharper rate (as if the true support were known) when the number of batches is large.

\subsection{A GLM example}\label{subsec: glm}
We illustrate the theoretical advantages of our framework through the example of generalized linear models (GLMs).
For each streaming batch $b \ge 1$, let the data be $D_b = \{ (X_i, Y_i) \}_{i \in \mathcal{I}_b}$, where $\mathcal{I}_b$ denotes the index set of observations in this batch, $X_i \in \mathbb{R}^{1 \times p}$ is a row vector of covariates, and $Y_i \in \mathbb{R}$ is the response variable. 
Assume $Y_i$ follows an exponential family distribution with a natural parameter $\zeta_i$:
\begin{equation*} 
    p_{\zeta_i}(Y_i ) = \exp\left( \frac{Y_i \zeta_i - g(X_i \zeta_i)}{a} + c(Y_i,a) \right),
\end{equation*}
where $g(\cdot)$ is a twice-differentiable function and we apply the canonical link $\zeta_i = X_i \beta^*$, where $\beta^* \in \mathbb{B}_0^p(s)$ is the sparse target vector.
For each batch $b$, we define the loss function $f_b$ and its gradient via the negative log-likelihood: 
$$
f_b(\beta) = \sum_{i \in \mI_b}\Big\{g(X_i \beta) - Y_i X_i \beta \Big\},
\quad \nabla f_b(\beta) = \sum_{i \in \mI_b} X_i^\top \Big\{g'(X_i \beta) - Y_i \Big\} \in \mathbb R^{p \times 1}.
$$ 
To establish theoretical guarantees for the streaming GLM learning, we impose the following regularity assumptions.

\begin{assumption}[Design matrices]\label{assump: cov}
For each batch $j\ge 1$ and each index $i \in \mI_j$, the observation $X_i \in \mathbb R^{1\times p}$ is independent and sub-Gaussian: each $X_{i} \overset{d}{=} Z_{i} (\Sigma^{(j)})^{1/2} $, where $Z_{i } = (Z_{i1}, \cdots, Z_{ip}) \in \mathbb R^{1 \times p}$ and each $Z_{ij}$ is i.i.d. centered 1-sub-Gaussian random vectors such that $\mathbf E ( Z_{i}^\top Z_{i} ) = I_p$.
The covariance matrix $\Sigma^{(j)} \in \mathbb R^{p \times p}$ satisfies
$$
K^{-1} \le  \inf_{j \ge 1}\Lambda_{\min}(\Sigma^{(j)} ) \le \sup_{j \ge 1}\Lambda_{\max}(\Sigma^{(j)} ) \le K,
$$
where $K >1$ is a universal constant.
\end{assumption}

\begin{assumption}[Restricted strong convexity, smoothness, and local Lipschitz]\label{assump: g}
The function $g(\cdot)$ is twice-differentiable. There exists a quadruple $(U, C, C_\delta, C_g)$ such that:
\begin{enumerate}
\item For every $X_i \in \mathbb R^{1 \times p}$ and $\beta, \beta' \in \mathbb B_0^p(C s)$,
\begin{equation}\label{eq: rss rsc assump}
 \frac{U^{-1}}2 |X_i(\beta - \beta')|_2^2 \le g(X_i \beta) - g(X_i \beta') - g'(X_i \beta') X_i(\beta - \beta') \le  \frac U2 |X_i(\beta - \beta')|_2^2. 
\end{equation}

\item For every $X_i \in \mathbb R^{1 \times p}$ $X_i$ and $\beta, \beta' \in \mathbb{B}_0^p(Cs) \cap \mathbb{B}_2^p \left(\beta^*, C_\delta \sqrt{\frac{s \log p}{N_1}} \right)$,
\begin{equation}\label{eq: lip assump}
 |g''(X_i \beta) - g''(X_i \beta')|  \le C_g |X_i (\beta - \beta')|.
\end{equation}
\end{enumerate}
\end{assumption}

Assumption \ref{assump: cov} ensures the (restricted) spectral stability and isotropy of design matrices across streaming batches and, together with the restricted strong convexity and smoothness of $g$ in \eqref{eq: rss rsc assump} (a condition standard in high-dimensional GLM literature, e.g., \citet{abramovich2016model}), guarantees the contraction of the iterative error in our algorithm.
This guarantee directly corresponds to Assumption~\ref{assump: rip} in the general analysis. 
The local Lipschitz condition \eqref{eq: lip assump} characterizes the stability of the model’s curvature near the target $\beta^*$, providing the technical basis for verifying Assumption~\ref{assump: lipschitz}. A similar condition is also required in renewable methods for high-dimensional GLMs \citep{luo2023GLM}.

The following theorem establishes the estimation results for learning GLM with streaming data.
\begin{theorem}[Streaming GLM]\label{th: glm}
Suppose Assumptions \ref{assump: cov} and \ref{assump: g} hold with the quadruple $(U, ~2C_s'+1,~ 4C_e'C_\beta' \sqrt{aKU},~ C_g)$. 
Assume the batch sample sizes satisfy $\sqrt{n_1} \ge C_{n} s (\log n_1) \cdot \log (p \vee n_1)$ for initialization, and $n_b \ge C_n (s \log p + \log b)$ for all subsequent batches $b \ge 2$.
Let the regularization parameters be set as
$$
\lambda_\beta^{(b,0)} \ge  \frac{1}{2\sqrt s}\|  \beta^* \|_2 ,
\quad \lambda_\beta^{(b,\infty)}  = 4C_\beta' \sqrt{aKU } \cdot \sqrt{\frac{\log(bp)}{N_b}},
$$ 
and assume the learning rate satisfies $\eta_b \in \left[ \frac{2KU}{(1+4K^2 U^2)N_b}, \frac{4KU}{(1+4K^2 U^2)N_b} \right]$ with a fixed decay rate $\kappa \in \left( \frac{4K^2 U^2}{1+ 4K^2 U^2}, 1 \right)$.
Then, with probability at least \(1 - 11p^{-2}\), the following \(\ell_2\) and \(\ell_0\) bounds hold for every batch \(b \ge 1\):
\begin{equation}\label{eq: resglm} 
\| \hat \beta^{(b)} - \beta^*\|_2 \le 4\sqrt{aUK}C_e' C_\beta' \cdot \sqrt{\frac{s(\log p + \log b)}{N_b}} ,
\quad \| \hat \beta^{(b)} \|_0 \le  (1+C_s')\cdot s, 
\end{equation} 
where $C_n, C_\beta', C_e'$, and $C_s'$ are universal constants depending only on $K,U, C_g$, and $\kappa$, independent of the batch index $b$.
\end{theorem}

The $\ell_2$ error bound in \eqref{eq: resglm} exhibits a \textbf{log-additive} dependence on the batch index $b$, i.e., \(C  \sqrt{ s (\log p+ \log b)/{N_b}}\), which significantly improves upon the \textbf{exponentially-multiplicative} dependence \(C^b \sqrt{ (s  \log p)/{N_b}}\) that appears in existing high-dimensional renewable methods \citep{luo2023GLM, Huang2024SIM}.
This provides two advantages: (i) the estimation error remains stable and non-divergent as $b$ increases, and (ii) for moderate sparsity $s \lesssim p^{c_1}$, our estimator achieves minimax rate-optimality \citep{abramovich2016model} uniformly across a vast number of batches $b \lesssim p^{C_2}$, for arbitrary constants $c_1 \in (0,1)$ and $C_2 > 0$, while the renewable method requires that $b \prec \log N_b$.
In addition, our framework imposes relatively mild sample size requirements per batch, avoiding the exponential growth ($n_b \asymp 2^b s^2(\log p + \log b)$) necessitated by methods like RADAR \citep{Agarwal2012RADAR, Luo2023inference}. 
Together, these features demonstrate that our method algorithmically enjoys favorable theoretical guarantees for long-term streaming applications.

Building on the estimation results in Theorem \ref{th: glm}, we further explore the statistical gains as data accumulate across streaming batches. 
Consider the following condition on the cumulative sample size $N_b$:
\begin{equation}\label{eq: sharp glm sample main}
 N_b \ge C_{N} \max\left\{ s^2 (\log^2  N_b) \cdot  \log^3 (p \vee N_b)  ,~ \frac{ \log p + \log b}{\min_{i \in\mS^*} | \beta_i^* |^2 } \right\},
\end{equation}
where $C_N$ is a universal constant depending solely on $a, K, U, C_g$, and $\kappa$.
This condition guarantees that $N_b$ is sufficiently large relative to both the weakest signal strength and the errors accumulated over previous batches.

\begin{theorem}[Oracle accuracy and support recovery]\label{th: glm sharper}
Suppose all conditions of Theorem \ref{th: glm} are satisfied. 
Then the following results hold:
\begin{enumerate}
\item With probability at least $1 - \varrho - 11p^{-2}$, we obtain the refined error bound
\begin{equation*} 
\| \hat \beta^{(b)} - \beta^*\|_2  \le C_{sharp}' \cdot \sqrt{\frac{s +\log (2b^2 / \varrho)}{N_b}} , \quad \text{for any }b \text{ satisfies } \eqref{eq: sharp glm sample main},
\end{equation*}
where $C_{sharp}'$ is a universal constants depending only on $a, K, U, C_g$, and $\kappa$.

\item If $s \succ1$, then with probability at least $1 - 12p^{-2}$, the estimated support $\hat{\mathcal{S}}^{(b)} = \text{supp}(\hat \beta^{(b)})$ further satisfies
\begin{equation*}
| \hat{\mS}^{(b)} \setminus \mS^* | + | \mS^* \setminus \hat{\mS}^{(b)}  | \prec s,
 \quad \text{for any }b \text{ satisfies } \eqref{eq: sharp glm sample main}.
\end{equation*}
\end{enumerate}
\end{theorem}

Theorem \ref{th: glm sharper} delivers two non-asymptotic advantages of our estimator:
\begin{itemize}
    \item As the batches increase, the estimation error adaptively refines from the minimax rate to the oracle rate $C\sqrt{s/N_b} $  (up to a mild logarithmic factor that depends on the confidence level and batch count).
    In other words, our estimator becomes as accurate as if one had full access to the entire dataset and knew the true support set in advance.
    Figure \ref{fig:rate-descend} in Section \ref{sec: intro} illustrates how the estimation accuracy adaptively improves with the number of batches. 
 
    \item Leveraging this adaptive accuracy, our method further achieves almost full support recovery, i.e., control over both false discovery and false negative proportions. 
    Consequently, the variable selection results become increasingly reliable as more batches are processed. 
\end{itemize}
To the best of our knowledge, the adaptive performance in both estimation and selection constitutes a novel contribution to the literature on high-dimensional online learning.

\section{Discussion}
This study introduces an asynchronous decomposition framework for online learning with streaming data, and develops a dynamic iterative hard thresholding procedure to produce sequential estimates.
Under this framework, the $\ell_2$ error bounds remain non-divergent as the batch index $b$ increases, and the required per-batch sample sizes do not grow geometrically with $b$. Furthermore, the estimation accuracy adaptively refines as batches accumulate, achieving the oracle estimation rate as if the entire dataset were accessed and the true support were known. These theoretical advantages are illustrated through a GLM example.

Several directions may further extend the present work. 
First, a natural progression is to establish asymptotic inference properties that hold uniformly for all batches $b \ge 1$. We believe this can be achieved within our proposed framework by additionally estimating the online precision matrix and applying suitable debiasing techniques. 
Second, our framework currently requires the per-batch loss to be twice differentiable. This condition is not met by losses such as the Huber or quantile loss. Therefore, it would be of great interest to introduce second-order smoothing for these losses and embed them into our framework, thereby further broadening the scope of our theory. 
Another important direction concerns adaptively collected data, where the observations in the current batch may depend on information from previous batches, as in multi-armed bandit settings. Recent studies address inference under such adaptive schemes \citep{2025WainDebiaseInfer, 2025wainSemiadapt}, and extending our high-dimensional framework to accommodate this type of dependence would be a meaningful direction for future work.

\newpage
\section{Proof of Proposition \ref{prop: 12} and Theorem \ref{th: minimax batch}}

We provide the proof of the general result below, beginning with several technical preliminaries.

\subsection{Decomposition of gradient update}
By the gradient descent procedure, it is direct to check that
\begin{equation}\label{eq: grad}
\begin{aligned}
\mathbb R^p \ni H^{(b, t+1)} :=& \hat \beta^{(b,t)} - \eta_b \nabla f_b( \hat \beta^{(b,t)} ) - \eta_b \sum_{j=1}^{b-1}  \nabla f_j( \hat \beta^{(j)})- \eta_b  \sum_{j=1}^{b-1}\left\{  \nabla^2 f_j( \hat \beta^{(j)}) \right\} ( \hat \beta^{(b,t)}  - \hat\beta^{(j)})\\
= & \bar \beta  + ( \hat \beta^{(b,t)}  - \bar \beta)  - \eta_b \left\{ \nabla f_b ( \hat \beta^{(b,t)} ) -  \nabla f_b (\bar \beta)  \right\} - \eta_b \sum_{j=1}^{b-1} \nabla^2 f_j (\hat \beta^{(j)}) ( \hat \beta^{(b,t)}  - \bar \beta) \\
&+ \eta_b \sum_{j=1}^{b-1}  \Big\{ \nabla f_j (\bar \beta )  - \nabla f_j (\hat \beta^{(j)}) -\nabla^2 f_j (\hat \beta^{(j)})( \bar \beta - \hat\beta^{(j)}) \Big\} 
- \eta_b \sum_{j=1}^b \nabla f_j (\bar \beta) \\
= & \bar \beta  - \underbrace{ \left\{ \eta_b A_b(\bar \beta,  \hat \beta^{(b,t)} ) + \eta_b \sum_{j=1}^{b-1} \nabla^2 f_j (\hat \beta^{(j)})  - I_p \right\} ( \hat \beta^{(b,t)} - \bar \beta) }_{\text{compression mapping}} \\
&+ \underbrace{  \eta_b \sum_{j=1}^{b-1}  \Big\{ \nabla f_j (\bar \beta )  - \nabla f_j (\hat \beta^{(j)}) -\nabla^2 f_j (\hat \beta^{(j)})( \bar \beta - \hat\beta^{(j)}) \Big\}}_{\text{asynchronous optimization error}}
- \eta_b \sum_{j=1}^b \nabla f_j (\bar \beta),
\end{aligned}
\end{equation}
where the last equality follows from the mean value theorem for vector-valued functions: 
\begin{equation}\label{eq: mean value}
\nabla f_j (y)  - \nabla f_j (x) = \underbrace{\int_0^1 \nabla^2 f_j(x + t(y-x)) \mathrm d t}_{=: A_j(x,y) \in \mathbb R^{p \times p} }  (y-x).
\end{equation}
Next, for arbitrary $\bar \beta \in \mathbb B_0^p(s)$ and every $b\ge 1$, by starting at $\hat \beta^{(b, t=0)} = \mathbf 0_p$ and taking 
$$
\begin{aligned}
    \lambda_\beta^{(b,0)} \ge  \frac{1}{2\sqrt s}\|  \bar \beta \|_2 ,
    \quad \lambda_\beta^{(b,\infty)} := C_\beta \cdot \frac{ \alpha_b }{ N_b} \ge \left(\frac{8}{m+M} \cdot \frac{m+M+M/\kappa}{m+M - M/\kappa} \right) \cdot \frac{\left\|  \sum_{j=1}^b \nabla f_j (\bar \beta) \right\|_{\infty} }{ N_b} ,
\end{aligned}
$$
where $C_\beta = \frac{8}{m+M} \cdot \frac{m+M+M/\kappa}{m+M - M/\kappa} $ and $\alpha_b \ge\left\|  \sum_{j=1}^b \nabla f_j (\bar \beta) \right\|_{\infty} $.
We aim to show that, for all $t \ge 0,~ b \ge 1$, we have 
\begin{equation}\label{eq: res}
\begin{aligned}
    \quad \| \hat \beta^{(b,t)} - \bar \beta\|_2 \le \underbrace{\frac{2(m+M) }{m+M - M/\kappa}}_{=: C_e >2} \cdot \sqrt s \lambda_\beta^{(b,t)},
    \quad   \| \hat \beta^{(b,t)}_{\bar \mS^c } \|_0 \le   \underbrace{(C_e-1)^2}_{=: C_s >1}\cdot s.
\end{aligned}
\end{equation}

\paragraph{Proof sketch} The main proof sketch applies the complete induction:

\begin{enumerate}

\item First we give a proof of \eqref{eq: res} in the case $b=1$ and $b=2$ (Sections \ref{sec: b=1} and \ref{sec: b=2} respectively). To extend this result to all $b \ge 1$, we employ strong mathematical induction on $b$.

\item Inductive Hypothesis (Outer): Assume that for some $b_0 \ge 2$, \eqref{eq: res} holds for all previous batches $b \le b_0$ and all $t \ge 0$. 

\item Inductive Step (Outer): We then prove that \eqref{eq: res} holds for the subsequent batch $b = b_0 + 1$ for all $t \ge 0$. We proceed this by mathematical induction on the iteration index $t$ (Sections \ref{sec: L0}):
    \begin{enumerate}
        \item Inductive Hypothesis (Inner): For $b = b_0 + 1$ and $t=0$, \eqref{eq: res} follows directly from the definitions of $\hat \beta^{(b_0+1,0)}$, $C_e$ and $C_s$.
        \item Inductive Step (Inner): For some $t_0 \ge 0$, assume \eqref{eq: res} holds for $b = b_0 + 1$ at iteration $t_0$, and then we verify that \eqref{eq: res} holds for $t = t_0+1, b= b_0+1$.
    \end{enumerate}
\end{enumerate} 
Finally, by applying the complete induction, we show that \eqref{eq: res} holds for every $t \ge 0$ and $b\ge 1$.

\subsection{In b=1 case}\label{sec: b=1}
For arbitrary $s$-sparse vector $\bar \beta \in \mathbb B_0^p(s)$, it is direct to check that
$$
\begin{aligned}
\mathbb R^p \ni H^{(b=1, t+1)} =& \hat \beta^{(1,t)} - \eta_1 \nabla f_1(\hat \beta^{(1,t)}) \\
= & \bar \beta  + ( \hat \beta^{(1,t)}- \bar \beta)  - \eta_1 \left\{ \nabla f_1 (\hat \beta^{(1,t)}) -  \nabla f_1 (\bar \beta)  \right\} -  \eta_1  \nabla f_1 (\bar \beta) \\
= & \bar \beta  - \underbrace{ \left\{ \eta_1 A_1(\bar \beta, \hat \beta^{(1,t)} ) - I_p \right\} ( \hat \beta^{(1,t)} - \bar \beta) }_{\text{compression mapping}} - \eta_1  \nabla f_1 (\bar \beta).
\end{aligned}
$$

For arbitrary $\bar \beta \in \mathbb B_0^p(s)$, by starting at $\hat \beta^{(1,t=0)} = \mathbf 0_p$ and taking 
$$
\begin{aligned}
    \lambda_\beta^{(1,0)} \ge  \frac{1}{2\sqrt s}\|  \bar \beta \|_2 ,
    \quad \lambda_\beta^{(1,\infty)} = C_\beta \cdot \frac{ \alpha_1}{ N_1}\ge \left(\frac{8}{m+M} \cdot \frac{m+M+M/\kappa}{m+M - M/\kappa} \right) \cdot \frac{\left\|  \nabla f_1 (\bar \beta) \right\|_{\infty} }{ N_1} ,
\end{aligned}
$$
we aim to show that, for all $t \ge 0$: 
\begin{equation}\label{eq: res batch 1}
\begin{aligned}
 \| \hat \beta^{(1,t)} - \bar \beta\|_2 \le  C_e \cdot \sqrt s \lambda_\beta^{(1,t)},
    \quad   \| \hat \beta^{(1,t)} _{\bar \mS^c } \|_0 \le  C_s \cdot s.
\end{aligned}
\end{equation}
where $C_e >2,~ C_s = (C_e-1)^2>1$ are two constants introduced in \eqref{eq: res}.
Indeed, in the case $t = 0$, we find \eqref{eq: res batch 1} holds.
Therefore, by mathematical induction, for a $t_0 \ge 0$, we assume \eqref{eq: res batch 1} holds for $t = t_0$, and we next prove \eqref{eq: res batch 1} still holds for $t = t_0+1$.

\paragraph{Sparsity}
Define $ {\mathcal{S}}^{(b, t)} := \left\{j\in [p] : \hat \beta^{(b,t)}_j \neq 0 \right\}$ and now we prove by contradiction that $| {\mathcal{S}}^{(1,t_0+1)} \backslash \mathcal{\bar S} | < C_s s$.
If $| {\mathcal{S}}^{(1,t_0+1)} \backslash \mathcal{\bar  S} | \ge C_s s$, we can construct a set $\mathcal S' \subseteq {\mathcal{S}}^{(1,t_0+1)} \backslash \mathcal{\bar  S}$ satisfying $|\mathcal S' | = C_s \cdot s$, and on the selected $\mS'$ we have
\begin{equation}\label{eq: sparse batch 1}
\begin{aligned}
    \sqrt{C_s s} \cdot \lambda_{\beta}^{(1,t_0+1)}
    \le & \sqrt{ \sum_{i \in \mathcal S'} \left(  H^{(1,t_0+1)}_i \right)^2 } \\
    \le & (1-\eta_1 m_1  ) \| \hat \beta^{(1,t_0)} - \bar \beta \|_2  + \eta_{1} \sqrt{C_s s} \left\| \nabla f_1 (\bar \beta) \right\|_{\infty} \\
    \le & (1- \eta_1 m_1  ) C_e \sqrt s \lambda_\beta^{( 1, t_0)} + \eta_1 \sqrt{ C_s s } \left\| \nabla f_1 (\bar \beta) \right\|_{\infty}\\
   < & \frac{M/\kappa}{m+M}  C_e \sqrt s \lambda_\beta^{(1, t_0+1)}  + 4 \eta_{1} \sqrt{ C_s s} \left\|  \nabla f_1 (\bar \beta) \right\|_{\infty},
\end{aligned}
\end{equation}
where the second inequality follows from the relationship
\begin{equation}
\begin{aligned}
&\sup_{\mS \subset [p], ~ |\mS|\le (1+2C_s)s}\left\| \left(\eta_1 A_1(\bar \beta, \hat \beta^{(1,t)} ) - I_p\right)_{\mS,\mS} \right\|_2\\
=&  \sup_{\mS \subset [p], ~ |\mS|\le (1+2C_s)s}\left\|   \left( \int_0^1 \eta_1 \nabla^2 f_j \left(\bar \beta + u(\hat \beta^{(1,t)}-\bar \beta) \right) \mathrm d u - I_p\right)_{\mS,\mS} \right\|_2 \\
\le &  \int_0^1  \sup_{\mS \subset [p], ~ |\mS|\le (1+2C_s)s} \left\| \left(\eta_1 \nabla^2 f_j \left(\bar \beta + u(\hat \beta^{(1,t)}-\bar \beta) \right) -I_p  \right)_{\mS,\mS} \right\|_2  \mathrm d u \\
\le & \max\left( |\eta_1 m n_1 -1|,   |\eta_1 M n_1 -1| \right)\\
\le & 1- \eta_1 m n_1 
~ \le ~ \frac{M}{m+M} \in (0,1),
\end{aligned}
\end{equation}
for arbitrary learning rate $\eta_1 \in \left[ \frac1{(m+M)N_1}, \frac2{(m+M)N_1} \right]$.
However, by the fact
\begin{equation}\label{eq: fact batch 1}
\begin{aligned}
\sqrt{C_s s} \cdot \lambda_{\beta}^{(1,t_0+1)} - \frac{M/\kappa}{m+M}  C_e \sqrt s \lambda^{(1, t_0+1)}
= & \sqrt{s} \lambda_{\beta}^{(1,t_0+1)}\\
\ge & \sqrt{s} \lambda_{\beta}^{(1,\infty)} \\
\ge & \left(\frac{8}{m+M} \cdot \frac{m+M+M/\kappa}{m+M - M/\kappa} \right) \cdot \frac{ \sqrt s \left\|  \nabla f_1 (\bar \beta) \right\|_{\infty} }{N_1}  \\
\ge & 4 \eta_1 (C_e -1) \sqrt s \left\|  \nabla f_1 (\bar \beta) \right\|_{\infty},
\end{aligned}
\end{equation}
it can be proved that there does not exist such a set $\mathcal{S}'$ satisfying \eqref{eq: sparse batch 1}. This implies that $|{\mathcal{S}}^{(1,t_0+1)} \backslash \bar{\mathcal{S}} | \le C_s s$.

\paragraph{Error bound}
Consider the decomposition 
\begin{equation}\label{eq: decom of beta}
   \hat \beta^{(1,t_0+1)}_i = H^{(1,t_0+1)}_i - H^{(1,t_0+1)}_i\mathbf{1}\left\{ | H^{(1,t_0+1)}_i| <\lambda_\beta^{(1,t_0+1)}\right\}, ~\text{ for every } i \in [p].
\end{equation}
We then have
\begin{equation}\label{eq: batch 1 error}
\begin{aligned}
&\left\| \hat \beta^{(1,t_0+1)}  - \bar \beta \right\|_2 \\
=& \sqrt{\sum_{i\in\mathcal{\bar S} } \left(H^{(1,t_0+1)}_i - \bar \beta_i - H^{(1,t_0+1)}_i\mathbf{1}\{ |H^{(1,t_0+1)}_i|<\lambda_\beta^{(1,t_0+1)}\}\right)^2 + \sum_{i\in{\mathcal{S}}^{(1,t_0+1)} \backslash \mathcal{ \bar S} } \left(H^{(1,t_0+1)}_i  \right)^2 }\\
\le&   \underbrace{ \sqrt{\sum_{i \in \mathcal{S}^{(1,t_0+1)} \cup \mathcal{\bar S} } \left(H^{(1,t_0+1)}_i - \bar \beta_i \right)^2 } }_{I} +  \underbrace{ \sqrt{\sum_{i \in\mathcal{\bar S} } \left(H^{(1,t_0+1)}_i\right)^2\mathbf{1}\{ |H^{(1,t_0+1)}_i| <\lambda_\beta^{(1,t_0+1)}\} } }_{II}  ,
\end{aligned}
\end{equation}
where 
\begin{align*}
    I <& (1- \eta_1 m_1  ) \|\beta^{(t,u+1)} - \beta^*\|_2
    +  \eta_1 \sqrt{ C_s+1 } \sqrt s\left\| \nabla f_1 (\bar \beta) \right\|_{\infty} \\
    \le & \frac{M/\kappa}{m+M} C_e \sqrt s \lambda_\beta^{(1, t_0+1)}
    + 4 \eta_1  \sqrt{ C_s s } \left\| \nabla f_1 (\bar \beta) \right\|_{\infty},
\end{align*}
which comes from the sparse property, \eqref{eq: sparse batch 1}, and $C_s\ge1$; and
\begin{align*}
    II < \sqrt{ \sum_{i \in \mathcal{\bar S} } (\lambda_\beta^{(1,t_0+1)})^2 } \le \sqrt{s}\lambda_\beta^{(1,t_0+1)}.
\end{align*}
Therefore we have 
$$
\begin{aligned}
\left\| \hat \beta^{(1,t_0+1)}  - \bar \beta \right\|_2
<& \sqrt s \lambda_{\beta}^{(1,t_0+1)} + \frac{M/\kappa}{m+M} C_e \sqrt s \lambda^{(1, t_0+1)} + 4 \eta_1 \sqrt{ C_s s } \left\| \nabla f_1 (\bar \beta) \right\|_{\infty}\\
\le & \sqrt s \lambda_{\beta}^{(1,t_0+1)} + \frac{M/\kappa}{m+M} C_e \sqrt s \lambda_\beta^{(1, t_0+1)} + \sqrt{C_s s} \cdot \lambda_{\beta}^{(1,t_0+1)} - \frac{M/\kappa}{m+M}  C_e \sqrt s \lambda_\beta^{(1, t_0+1)}\\
=& C_e \sqrt s \lambda_{\beta}^{(1,t_0+1)},
\end{aligned}
$$
where the second inequality follows from the fact \eqref{eq: fact batch 1}.
Therefore, we prove that \eqref{eq: res batch 1} still holds in the $(t_0+1)$-th iteration.
By mathematical induction, it implies that \eqref{eq: res batch 1} holds for every $t\ge 0$ in the first batch. 

Finally, we terminate the iterations at a sufficiently large index $t_1^*$ such that $\lambda^{(1, t_1^*)} = \lambda^{(1, \infty)}$. 
We define $\hat \beta^{(1)} := \hat \beta^{(1, t_1^*)}$ as the first-batch estimator, which satisfies
\begin{equation*}
    \quad \| \hat \beta^{(1)} - \bar \beta\|_2 \le C_e \cdot \sqrt s \lambda_\beta^{(1,\infty)},
    \quad   \| \hat \beta^{(1)} _{\bar \mS^c } \|_0 \le C_s \cdot s.
\end{equation*}

\subsection{In b=2 case}\label{sec: b=2}
The proof in $b=2$ is quite similar to Section \ref{sec: b=1}.
From \eqref{eq: grad} we write the decomposition
$$
\begin{aligned}
H^{(b=2, t+1)} =& \hat \beta^{(2,t)} - \eta_2 \nabla f_2( \hat \beta^{(2,t)} ) - \eta_2 \nabla f_1( \hat \beta^{(1)})- \eta_2  \nabla^2 f_1( \hat \beta^{(1)}) ( \hat \beta^{(2,t)}  - \hat\beta^{(1)})\\ 
= & \bar \beta  - \left\{ \eta_2 A_2(\bar \beta,  \hat \beta^{(2,t)} ) + \eta_2 \nabla^2 f_1 (\hat \beta^{(1)})  - I_p \right\} ( \hat \beta^{(2,t)} - \bar \beta)  \\
&+ \eta_2  \left\{ \nabla f_1 (\bar \beta )  - \nabla f_1 (\hat \beta^{(1)}) -\nabla^2 f_j (\hat \beta^{(1)})( \bar \beta - \hat\beta^{(1)}) \right\} 
- \eta_2 \sum_{j=1}^2 \nabla f_j (\bar \beta)
\end{aligned}
$$
By starting at $\hat \beta^{(2,t=0)} = \mathbf 0_p$ and taking 
$$
\begin{aligned}
    \lambda_\beta^{(2,0)} \ge  \frac{1}{2\sqrt s}\|  \bar \beta \|_2 ,
    \quad \lambda_\beta^{(2,\infty)} = C_\beta \frac{\alpha_2}{N_2}  \ge C_\beta \frac{\left\| \sum_{j=1}^2 \nabla f_j (\bar \beta) \right\|_{\infty} }{ N_2} ,
\end{aligned}
$$
where $C_\beta = \frac{8}{m+M} \cdot \frac{m+M+M/\kappa}{m+M - M/\kappa}$ is a universal constant, we aim to show that, for all $t \ge 0$:
\begin{equation}\label{eq: res batch 2}
\begin{aligned}
    \quad \| \hat \beta^{(2,t)} - \bar \beta\|_2 \le  C_e \cdot \sqrt s \lambda_\beta^{(2,t)},
    \quad   \| \hat \beta^{(2,t)}_{\bar \mS^c } \|_0 \le  C_s \cdot s.
\end{aligned}
\end{equation}
Indeed, in the case $t = 0$, we find \eqref{eq: res batch 2} holds.
Therefore, we assume \eqref{eq: res batch 2} holds in $t = t_0$ for some $t_0 \ge 0$, and next prove \eqref{eq: res batch 2} still holds for $t = t_0+1$.

\paragraph{Sparsity} Similar to \eqref{eq: sparse batch 1}, if $| {\mathcal{S}}^{(2,t_0+1)} \backslash \mathcal{\bar  S} | \ge C_s s$, we can construct a set $\mathcal S' \subseteq {\mathcal{S}}^{(2,t_0+1)} \backslash \mathcal{\bar  S}$ satisfying $|\mathcal S' | = C_s \cdot s$, and then 
\begin{equation}\label{eq: sparse batch 2}
\begin{aligned}
& \sqrt{C_s s} \cdot \lambda_{\beta}^{(2,t_0+1)}
    \le \sqrt{ \sum_{i \in \mathcal S'} \left(  H^{(2,t_0+1)}_i \right)^2 } \\
\overset{(i)}\le & \frac{M}{m+M} \| \hat \beta^{(2,t_0)} - \bar \beta \|_2  + \eta_2 \left\|  \left\{ \nabla f_1 (\bar \beta )  - \nabla f_1 (\hat \beta^{(1)}) -\nabla^2 f_j (\hat \beta^{(1)})( \bar \beta - \hat\beta^{(1)}) \right\}_{\mathcal S'} \right\|_2 + \eta_2 \sqrt{C_s s }\left\| \sum_{j=1}^2 \nabla f_j (\bar \beta) \right\|_{\infty} \\
\overset{(ii)}\le& \frac{M/\kappa}{m+M} C_e \cdot \sqrt s \lambda_\beta^{(2,t_0+1)}
+ \eta_2 L_1 C_e^2 s (\lambda_\beta^{(1, \infty)})^2
+\eta_2 \sqrt{C_s s}\left\| \sum_{j=1}^2 \nabla f_j (\bar \beta) \right\|_{\infty} \\
\overset{(iii)}\le& \frac{M/\kappa}{m+M} C_e \cdot \sqrt s \lambda_\beta^{(2,t_0+1)}
+ \frac{\sqrt s}2 \lambda_\beta^{(2,\infty)}
+\eta_2 \sqrt{C_s s}\left\| \sum_{j=1}^2 \nabla f_j (\bar \beta) \right\|_{\infty},  
\end{aligned}
\end{equation}
where:
\begin{itemize}
    \item The inequality (i) follows from the relationship (comes from \eqref{eq: rip} and \eqref{eq: mean value}): 
\begin{equation}\label{eq: b=2 rip}
\begin{aligned}
& \sup_{\mS \subset [p], ~ |\mS|\le (1+2C_s)s}\left\| \left( \eta_2 A_2(\bar \beta,  \hat \beta^{(2,t)}   ) + \eta_2 \nabla^2 f_1 (\hat \beta^{(1)})  - I_p \right)_{\mS,\mS}  \right\|_2\\
\le &  \int_0^1  \sup_{\mS \subset [p], ~ |\mS|\le (1+2C_s)s} \left\| \left( \eta_2 \nabla^2 f_2 \left(\bar \beta + u (\hat \beta^{(2,t)}-\bar \beta) \right) + \eta_2 \nabla^2 f_1(\hat \beta^{(1)}) -I_p  \right)_{\mS,\mS} \right\|_2  \mathrm d u \\
\le & \max\left( |M(n_1 + n_2) \eta_2  -1|,   |m(n_1 + n_2)\eta_2 -1| \right)\\
\le & 1-  \eta_2 m N_2 
~ \le ~ \frac{M}{m+M} \in (0,1),
\end{aligned}
\end{equation}
for arbitrary learning rate $\eta_2 \in \left[ \frac1{(m+M)N_2}, \frac2{(m+M)N_2} \right]$.
We also use the notation $\mS_2':= \mS'\cup \bar \mS\cup \mS^{(1)}$, and the result in Section \ref{sec: b=1} leads that $| \mS_2'| \le (2C_s +1)s$.

\item The inequality (ii) follows from \eqref{eq: res batch 2} (assumed to hold at the $t_0$-th iteration on the second-batch learning, and the fact $\lambda_\beta^{(2,t_0+1)} \ge \kappa  \lambda_\beta^{(2,t_0 )} $), together with the first-batch learning result \eqref{eq: res batch 1}.
Additionally, it follows the result
\begin{equation}\label{eq: b=2 lip}
\begin{aligned}
\sup_{\mS \subset [p]: ~ |\mS|\le (1+C_s)s} \left\|  \left\{ \nabla f_1 (\bar \beta )  - \nabla f_1 (\hat \beta^{(1)}) -\nabla^2 f_j (\hat \beta^{(1)})( \bar \beta - \hat\beta^{(1)}) \right\}_{\mathcal S } \right\|_2
\le &  L_1\left\| \bar \beta - \hat \beta^{(1)}\right\|_2^2 \\
\le &L_1 C_e^2 s (\lambda_\beta^{(1, \infty)})^2,
\end{aligned}
\end{equation}
where the first inequality comes from the restricted Lipschitz Assumption \ref{assump: lipschitz}, and the last comes from \eqref{eq: res batch 1}.

\item The inequality (iii) follows from Assumption \ref{assump: N1}, specifically, assume that
$$
\frac{4C_e^2 C_\beta}{m+M} \cdot\frac{s L_1 \alpha_1^2 }{N_1^2} 
\le \sqrt s \alpha_2,
$$
which leads to
\begin{equation}\label{eq: batch 2 scale}
\eta_2 L_1 C_e^2 s (\lambda_\beta^{(1, \infty)})^2 \le \frac{2}{(m+M)N_2} C_e^2 C_\beta^2 \cdot \frac{s L_1 \alpha_1^2 }{N_1^2} 
\le \frac{\sqrt s \alpha_2 C_\beta}{2N_2} = \frac{\sqrt s}2 \lambda_\beta^{(2,\infty)}. 
\end{equation}

\end{itemize}

On the other hand, based on the definitions of $C_s$ and $C_e$ (introduced in \eqref{eq: res}), we establish the following fact in a manner analogous to \eqref{eq: fact batch 1}:
\begin{equation}\label{eq: fact batch 2}
\begin{aligned}
&\sqrt{C_s s} \cdot \lambda_{\beta}^{(2,t_0+1)} - \frac{M/\kappa}{m+M}  C_e \sqrt s \lambda^{(2, t_0+1)}
\ge  \sqrt{s} \lambda_{\beta}^{(2,\infty)} \\
\ge & \frac{\sqrt{s}}2 \lambda_{\beta}^{(2,\infty)} + \left(\frac{4}{m+M} \cdot \frac{m+M+M/\kappa}{m+M - M/\kappa} \right) \cdot \frac{\sqrt s \left\| \sum_{j=1}^2 \nabla f_j (\bar \beta) \right\|_{\infty}}{N_2}  \\
\ge & \frac{\sqrt{s}}2 \lambda_{\beta}^{(2,\infty)} + 2 \eta_2 \sqrt{C_s }\cdot  \sqrt s\left\|\sum_{j=1}^2 \nabla f_j(\bar \beta) \right\|_{\infty}.
\end{aligned}
\end{equation}
A contradiction between \eqref{eq: sparse batch 2} and \eqref{eq: fact batch 2} arises, demonstrating that $|{\mathcal{S}}^{(2, t_0+1)} \backslash \bar{\mathcal{S}} | \le C_s s$.

\paragraph{Error bound}
Similar to \eqref{eq: batch 1 error}, we have
\begin{equation}\label{eq: batch 2 error}
\begin{aligned}
&\left\| \hat \beta^{(2,t_0+1)}  - \bar \beta \right\|_2 \\
\le&  \sqrt{\sum_{i \in \mathcal{S}^{(2,t_0+1)} \cup \mathcal{\bar S} } \left(H^{(2,t_0+1)}_i - \bar \beta_i \right)^2 }  +  \sqrt{\sum_{i \in\mathcal{\bar S} } \left(H^{(2,t_0+1)}_i\right)^2\mathbf{1} \left\{ |H^{(2,t_0+1)}_i| <\lambda_\beta^{(2,t_0+1)} \right\} }\\
\le & \frac{M/\kappa}{m+M} C_e \cdot \sqrt s \lambda_\beta^{(2,t_0+1)}  + \eta_2 L_1 C_e^2 s (\lambda_\beta^{(1, \infty)})^2 +  \eta_2  \sqrt{ C_s+1 } \cdot \sqrt s\left\| \sum_{j=1}^2 \nabla f_j (\bar \beta) \right\|_{\infty}
+ \sqrt s \lambda_\beta^{(2,t_0+1)}  \\
\le & \left( 1+ \frac{M/\kappa}{m+M} C_e \right) \cdot \sqrt s \lambda_\beta^{(2,t_0+1)} + \frac{\sqrt s}2 \lambda_\beta^{(2,\infty)} + 2 \sqrt{C_s s} \cdot \eta_2  \left\| \sum_{j=1}^2 \nabla f_j (\bar \beta) \right\|_{\infty}\\
\le & C_e \sqrt s \lambda_\beta^{(2,t_0+1)},
\end{aligned}
\end{equation}
where the second inequality follows the same scaling as \eqref{eq: sparse batch 2}, the third inequality follows \eqref{eq: batch 2 scale}, and the last inequality again relies on the relationship \eqref{eq: fact batch 2} (recall $1+ \sqrt{C_s} = C_e$).
Therefore, we prove that \eqref{eq: res batch 2} still holds in the $(t_0+1)$-th iteration.
By mathematical induction, it implies that \eqref{eq: res batch 2} holds for every $t\ge 0$ in the second-batch learning ($b=2$). 

We terminate the iterations at a sufficiently large index $t_2^*$ such that $\lambda^{(2, t_2^*)} = \lambda^{(2, \infty)}$, and define $\hat \beta^{(2)} := \hat \beta^{(2, t_2^*)}$ as the second-batch estimator, which satisfies
\begin{equation*}
    \quad \| \hat \beta^{(2)} - \bar \beta\|_2 \le C_e \cdot \sqrt s \lambda_\beta^{(2,\infty)},
    \quad   \| \hat \beta^{(2)}_{\bar \mS^c } \|_0 \le C_s \cdot s.
\end{equation*}
Combining Section \ref{sec: b=1} and Section \ref{sec: b=1}, we complete the proof of Proposition \ref{prop: 12}.

\subsection{In general b case}\label{sec: L0}
Sections \ref{sec: b=1} and \ref{sec: b=2} verified \eqref{eq: res} for the base cases $b=1$ and $b=2$ (for all $t \ge 0$). To generalize this to all $b \ge 1$, we invoke strong induction: assume \eqref{eq: res} holds for all batches up to $b_0$ (i.e., for all $b \le b_0$ where $b_0 \ge 2$), and we aim to prove that \eqref{eq: res} still holds for the next batch $b = b_0 + 1$. 

In the $(b_0+1)$-th batch-learning, we apply a secondary induction on $t$: For the initial iteration $t=0$, \eqref{eq: res} holds trivially, like the analysis in Section \ref{sec: b=1}. Assuming the result holds at the $t_0$-th iteration for some $t_0 \ge 0$, we then show it remains valid for the $(t_0 + 1)$-th iteration.

\paragraph{Sparsity} First we prove by contradiction that $| {\mathcal{S}}^{(b_0+1,t_0+1)} \backslash \mathcal{\bar S} | < C_s s$, where recall $C_s >1$.
If not so, we can construct a set $\mathcal S' \subseteq {\mathcal{S}}^{(b_0+1,t_0+1)} \backslash \mathcal{\bar  S}$ satisfying $|\mathcal S' | = C_s \cdot s$.
Then by the decomposition \eqref{eq: grad}, similar to \eqref{eq: sparse batch 2}, we have
\begin{equation}\label{eq: sparse batch b}
\begin{aligned}
& \sqrt{C_s s} \cdot \lambda_{\beta}^{(b_0+1,t_0+1)}
\le \sqrt{ \sum_{i \in \mathcal S'} \left(  H^{(b_0+1,t_0+1)}_i \right)^2 } \\
\overset{(i)}\le & \frac{M}{m+M} \| \hat \beta^{(b_0+1,t_0)} - \bar \beta \|_2  +  \eta_{b_0+1} \sum_{j=1}^{b_0} \left\| \Big\{ \nabla f_j (\bar \beta )  - \nabla f_j (\hat \beta^{(j)}) -\nabla^2 f_j (\hat \beta^{(j)})( \bar \beta - \hat\beta^{(j)}) \Big\}_{\mathcal S'} \right\|_2\\ 
& + \eta_{b_0+1} \sqrt{C_s s}\left\| \sum_{j=1}^{b_0+1} \nabla f_j (\bar \beta) \right\|_{\infty} \\
\overset{(ii)}\le& \frac{M/\kappa}{m+M} C_e \cdot \sqrt s \lambda_\beta^{(b_0+1 ,t_0+1)}
+ \frac{\sqrt s}2 \cdot \lambda_\beta^{(b_0+1, \infty)}
+  \eta_{b_0+1} \sqrt{C_s s}\left\| \sum_{j=1}^{b_0+1} \nabla f_j (\bar \beta) \right\|_{\infty} \\
\end{aligned}
\end{equation}
where:
\begin{itemize}
    \item Inequality (i) follows in a similar manner to \eqref{eq: b=2 rip}: 
\begin{equation}\label{eq: batch b rip}
\begin{aligned}
& \sup_{ |\mS|\le (1+2C_s)s}\left\| \left( \eta_{b_0+1} A_{b_0+1}(\bar \beta,  \hat \beta^{(b_0+1, t_0)}   ) + \eta_{b_0+1} \sum_{j=1}^{b_0} \nabla^2 f_j (\hat \beta^{(j)})  - I_p \right)_{\mS,\mS}  \right\|_2\\
\le & \sup_{\substack{u \in [0,1] \\ |\mS|\le (1+2C_s)s}} \left\| \left( \eta_{b_0+1} \nabla^2 f_{b_0+1} \left(\bar \beta+ u(\hat \beta^{(b_0+1, t_0)} - \bar \beta) \right) + \eta_{b_0+1} \sum_{j=1}^{b_0} \nabla^2 f_j (\hat \beta^{(j)})  - I_p \right)_{\mS,\mS}  \right\|_2 \\
\le & \max\left( \left|-1+  M \eta_{b_0+1}\sum_{j=1}^{b_0+1} n_j  \right|, \left|-1+  m \eta_{b_0+1}\sum_{j=1}^{b_0+1} n_j  \right|\right)\\
\le & 1-  \eta_{b_0+1} m N_{b_0+1}
~ \le ~ \frac{M}{m+M} \in (0,1),
\end{aligned}
\end{equation}
for arbitrary learning rate $ \eta_{b_0+1} \in \left[ \frac1{(m+M)N_{b_0+1} }, \frac2{(m+M)N_{b_0+1} } \right]$.

\item Inequality (ii) follows from \eqref{eq: res} (assumed to hold at the $t_0$-th iteration on the $(b_0+1)$-th batch).
Additionally, it follows similarly to \eqref{eq: b=2 lip} that:
\begin{equation}\label{eq: upsilon}
\begin{aligned}
& \eta_{b_0+1}\sum_{j=1}^{b_0} \left\| \Big\{ \nabla f_j (\bar \beta )  - \nabla f_j (\hat \beta^{(j)}) -\nabla^2 f_j (\hat \beta^{(j)})( \bar \beta - \hat\beta^{(j)}) \Big\}_{\mathcal S'} \right\|_2\\
\le & \eta_{b_0+1} \sum_{j=1}^{b_0}  L_j \left\| \bar \beta - \hat \beta^{(j)}\right\|_2^2 \\
~\le& \frac{2 C_e^2 C_\beta^2}{(m+M) N_{b_0+1}} \sum_{j=1}^{b_0}  \frac{L_j  s \alpha_j^2 }{N_j^2}\\
\le& \frac{\sqrt s}2 \cdot \lambda_\beta^{(b_0+1, \infty)},
\end{aligned}
\end{equation}
where the first inequality comes from Assumption \ref{assump: lipschitz}, and the last comes from Assumption \ref{assump: N1}.

\end{itemize}

On the other hand, based on the fact $C_e - 1 =\sqrt{C_s}$, we establish the following fact in a manner analogous to \eqref{eq: fact batch 2}:
\begin{equation}\label{eq: fact batch b}
\begin{aligned}
&\sqrt{C_s s} \cdot \lambda_{\beta}^{(b_0+1,t_0+1)} - \frac{M/\kappa}{m+M}  C_e \sqrt s \lambda^{(b_0+1, t_0+1)}\\
\ge & \sqrt{s} \lambda_{\beta}^{(b_0+1,t_0+1)}\\
\ge &\frac{\sqrt s}2 \lambda_\beta^{(b_0+1, \infty)}
+  2 \eta_{b_0+1} (C_e -1) \sqrt{ s} \alpha_{b_0 +1}\\
\ge &\frac{\sqrt s}2 \lambda_\beta^{(b_0+1, \infty)}
+  2 \eta_{b_0+1} \sqrt{C_s s} \left\| \sum_{k=1}^{b_0+1} \nabla f_k (\bar \beta) \right\|_{\infty}.
\end{aligned}
\end{equation}
A contradiction between \eqref{eq: sparse batch b} and \eqref{eq: fact batch b} arises, demonstrating that $|{\mathcal{S}}^{(b_0+1, t_0+1)} \backslash \bar{\mathcal{S}} | \le C_s s$. 

\paragraph{Error bound}
Similar to \eqref{eq: batch 2 error}, we have
\begin{equation}\label{eq: batch b error}
\begin{aligned}
&\left\| \hat \beta^{(b_0+1,t_0+1)}  - \bar \beta \right\|_2 \\
\le&  \sqrt{\sum_{i \in \mathcal{S}^{(b_0+1,t_0+1)} \cup \mathcal{\bar S} } \left(H^{(b_0+1,t_0+1)}_i - \bar \beta_i \right)^2 }  +  \sqrt{\sum_{i \in\mathcal{\bar S} } \left(H^{(b_0+1,t_0+1)}_i\right)^2\mathbf{1} \left\{ |H^{(b_0+1,t_0+1)}_i| <\lambda_\beta^{(b_0+1,t_0+1)} \right\} }\\
\le & \frac{M/\kappa}{m+M} C_e \cdot \sqrt s \lambda_\beta^{(b_0+1,t_0+1)}  + {C_e^2 C_{\beta}^2} \cdot \eta_{b_0+1}  \sum_{j=1}^{b_0} \frac{s L_j \alpha_j^2}{N_j^2}   \\
&+  \eta_{b_0+1} \sqrt{ C_s+1 } \cdot \sqrt s \left\| \sum_{j=1}^{b_0+1} \nabla f_j (\bar \beta) \right\|_{\infty} + \sqrt s \lambda_\beta^{(b_0+1,t_0+1)}  \\
\le & \left( 1+ \frac{M/\kappa}{m+M} C_e \right) \cdot \sqrt s \lambda_\beta^{(b_0+1,t_0+1)} +
\frac{\sqrt s}2  \lambda_\beta^{(b_0+1, \infty)} +  \eta_{b_0+1} \sqrt{(C_s+1) s}\left\| \sum_{k=1}^{b_0+1} \nabla f_k (\bar \beta) \right\|_{\infty} \\
\le & C_e \sqrt s \lambda_\beta^{(b_0+1,t_0+1)},
\end{aligned}
\end{equation}
where the second inequality follows the same scaling as in \eqref{eq: sparse batch b}, the third inequality follows \eqref{eq: upsilon}, and the last inequality relies on the relationship \eqref{eq: fact batch b} and $C_e - 1 = \sqrt{C_s} \ge 1$.
Therefore, we prove that \eqref{eq: res} still holds in the $(t_0+1)$-th iteration of the $(b_0+1)$-th batch learning.
We terminate the iterations at a sufficiently large index $t_{b_0+1}^*$ such that $\lambda_\beta^{(b_0+1, t_{b_0+1}^*)} = \lambda_\beta^{(b_0+1, \infty)}$, and define $\hat \beta^{(b_0+1)} := \hat \beta^{(b_0+1, t_{b_0+1}^*)}$ as the $(b_0+1)$-th batch estimator, which satisfies
\begin{equation*}
\| \hat \beta^{(b_0+1)} - \bar \beta\|_2 \le C_e \cdot \sqrt s \lambda_\beta^{(b_0+1,\infty)},
    \quad   \| \hat \beta^{(b_0+1)}_{\bar \mS^c } \|_0 \le C_s \cdot s.
\end{equation*}
By mathematical induction, it implies that \eqref{eq: res} holds for every $t\ge 0$ and $b \ge 1$, and therefore we complete the proof of Theorem \ref{th: minimax batch}.

\section{Proof of Theorem \ref{th: sharper batch}}
Based on the results derived from Theorem \ref{th: minimax batch}, we proceed to a more refined analysis of the accuracy of $\ell_2$ error rate.
For ease of display, we use the abbreviations
$$
\begin{aligned}
\Phi^{(b,t)} :=& \eta_b A_b(\bar \beta,  \hat \beta^{(b,t)} ) + \eta_b \sum_{j=1}^{b-1} \nabla^2 f_j (\hat \beta^{(j)})  - I_p  \in \mathbb R^{p\times p}, \\
\Upsilon^{(j)} := & \nabla f_j (\bar \beta )  - \nabla f_j (\hat \beta^{(j)}) -\nabla^2 f_j (\hat \beta^{(j)})( \bar \beta - \hat\beta^{(j)}) \in \mathbb R^p,\\
\Xi^{(b)} := &   \eta_b  \sum_{j=1}^b \nabla f_j (\bar \beta) \in \mathbb R^p.
\end{aligned}
$$
Recall 
\begin{equation}\label{eq: b*}
\begin{aligned}
  b_1^*:=& \inf\left\{ b \ge 1: ~\min_{i \in \bar S} |\bar \beta_i| \ge \left( \frac{2}{m+M} + \frac{m+M}2 C_\beta^2 \right) \cdot \frac{ \alpha_b }{N_b} \right\},\\
b_2^*:=& \inf\left\{ b \ge 2: ~C_e C_\beta^2  \sum_{j=1}^{b-1} \frac{ s L_j \alpha_j^2 }{N_j^2}  \le  \theta_b \right\},\\
b^* :=& b_1^* \vee b_2^*.
\end{aligned}
\end{equation}
And in the GLM setting, this can be achieved when $N_b \gtrsim \max\left\{ s^2 (\log^2N_b ) \cdot \log^3(p \vee N_b),~ \frac{\log p + \log b }{\min_{i \in \mS^*} |\beta_i^*|^2 }  \right\}$, see \eqref{eq: sharp glm sample}.
For the first $b^*-1$ batches ($1\le b \le b^*-1$), the results are the same as in Theorem \ref{th: minimax batch}:
\begin{equation*} 
\begin{aligned}
\| \hat \beta^{(b)} - \bar \beta\|_2 \le \underbrace{\frac{2(m+M) }{m+M - M/\kappa}}_{=: C_e >2} \cdot \underbrace{ C_\beta \frac{  \sqrt s  \alpha_b}{ N_b} }_{=: \sqrt s \lambda_\beta^{(b,\infty)}} ,
    \quad   \| \hat \beta^{(b,t)}_{\bar \mS^c } \|_0 \le   \underbrace{(C_e-1)^2}_{=: C_s >1}\cdot s.
\end{aligned}
\end{equation*}

Now consider the $\ell_2$ error at the batch $b^*$ and all subsequent batches. 
Fix an arbitrary batch $b \ge b^*$. Suppose that after $t_b$ iterations within this batch, the threshold $\lambda_\beta^{(b,t)}$ first attains the level $\lambda_\beta^{(b,\infty)}$. 
Thereafter, the algorithm keeps the threshold fixed (at $\lambda_\beta^{(b,\infty)}$) and performs several additional hard-thresholding iterations. 
We next refine the error bound attained after this stage.

For $t \ge t_b$, from \eqref{eq: batch b error}, we learn that the updated error can be decomposed as
\begin{equation*}
\begin{aligned}
&\| \hat \beta^{(b ,t+1)} - \bar \beta\|_2^2 \\
=& \sum_{i \in \mathcal{\bar S} } \left\{  \langle \Phi_{\cdot i}^{(b,t)},  \bar \beta - \hat \beta^{(b,t)}  \rangle + \eta_b \sum_{j=1}^{b-1} \Upsilon_i^{(j)}  - \Xi^{(b)}_i  -H^{(b,t+1)}_i \mathbf{1} \left( |H^{(b,t+1)}_i| <\lambda_\beta^{(b,\infty)} \right) 
  \right\}^2\\
&+\sum_{i \in \mS^{(b,t+1)}\setminus \mathcal{\bar S} } \left\{ \langle \Phi_{\cdot i}^{(b,t)},  \bar \beta - \hat \beta^{(b,t)}  \rangle + \eta_b \sum_{j=1}^{b-1} \Upsilon_i^{(j)}  - \Xi^{(b)}_i \right\}^2,
\end{aligned}
\end{equation*}
leading that
\begin{equation}\label{eq: refine b b}
\begin{aligned}
\left\| \hat \beta^{(b,t+1)}  - \bar \beta \right\|_2 
\le&  \sqrt{\sum_{i \in \mathcal{S}^{(b,t+1)} \cup \mathcal{\bar S} } \left( \langle \Phi_{\cdot i}^{(b,t)},  \bar \beta - \hat \beta^{(b,t)}  \rangle + \eta_b \sum_{j=1}^{b-1} \Upsilon_i^{(j)}\right)^2 }  + \|\Xi_{\mathcal{\bar S}}^{(b)} \|_2\\
& +  \sqrt{\sum_{i \in\mathcal{\bar S} } \left(H^{(b,t+1)}_i\right)^2\mathbf{1} \left\{ |H^{(b,t+1)}_i| <\lambda_\beta^{(b,\infty)} \right\}
+ \sum_{i \in \mS^{(b,t+1)}\setminus \mathcal{\bar S} } \left( \Xi^{(b)}_i \right)^2}\\
\le & \frac{M}{m+M}\| \hat \beta^{(b,t)} - \bar \beta \|_2 + \eta_b\sum_{j=1}^{b-1}\left\| \Upsilon^{(j)}_{\mathcal{S}^{(b,t+1)} \cup \mathcal{\bar S}} \right\|_2  + \|\Xi_{\mathcal{\bar S}}^{(b)} \|_2\\
& +  \sqrt{\sum_{i \in\mathcal{\bar S} } \left(H^{(b,t+1)}_i\right)^2\mathbf{1} \left\{ |H^{(b,t+1)}_i| <\lambda_\beta^{(b,\infty)} \right\}
+ \sum_{i \in \mS^{(b,t+1)}\setminus \mathcal{\bar S} } \left( \Xi^{(b)}_i \right)^2},
\end{aligned}
\end{equation}
where the last inequality follows from \eqref{eq: batch b rip} and the fact $\| \hat\beta^{(b,t)}_{\bar \mS^c } \|_0 \le C_s s$ for all batches $b \ge 1$ and iterations $t \ge 0$ (as proved in Theorem 1).
Now, by the definition of $b^*$, on the set $\mathcal{\bar S}$ we have
\begin{equation}\label{eq: sa1}
\begin{aligned}
&\sum_{i \in\mathcal{\bar S} } \left(H^{(b,t+1)}_i\right)^2\mathbf{1} \left\{ |H^{(b,t+1)}_i| <\lambda_\beta^{(b,\infty)} \right\}\\
\le&\sum_{i \in\mathcal{\bar S} } \left(\lambda_\beta^{(b,\infty)}\right)^2 \cdot \mathbf{1} \left\{ |\bar \beta_i| - \left|\langle \Phi_{\cdot i}^{(b,t)},  \bar \beta - \hat \beta^{(b,t)}  \rangle + \eta_b \sum_{j=1}^{b-1} \Upsilon_i^{(j)} \right| - |\Xi^{(b)}_i| <\lambda_\beta^{(b,\infty)} \right\}\\
\le&\sum_{i \in\mathcal{\bar S} } \left(\lambda_\beta^{(b,\infty)}\right)^2 \cdot \mathbf{1} \left\{ \left( \frac{(m+M)C_\beta}2-1\right)\cdot 
 \lambda_\beta^{(b,\infty)}  \le \left|\langle \Phi_{\cdot i}^{(b,t)},  \bar \beta - \hat \beta^{(b,t)}  \rangle + \eta_b \sum_{j=1}^{b-1} \Upsilon_i^{(j)} \right|  \right\}\\
 \le&  \frac1{\left( \frac{(m+M)C_\beta}2-1  \right)^2} \sum_{i \in\mathcal{\bar S} } \left|\langle \Phi_{\cdot i}^{(b,t)},  \bar \beta - \hat \beta^{(b,t)}  \rangle + \eta_b \sum_{j=1}^{b-1} \Upsilon_i^{(j)} \right|^2.
\end{aligned}
\end{equation}
Additionally, by the definition of $\mS^{(b,t+1)}$ and $\lambda_\beta^{(b, \infty)}$, on the set $\mS^{(b,t+1)}\setminus \mathcal{\bar S}$ we have
\begin{equation}\label{eq: sa2-1}
\begin{aligned}
& \left|\langle \Phi_{\cdot i}^{(b,t)},  \bar \beta - \hat \beta^{(b,t)}  \rangle + \eta_b \sum_{j=1}^{b-1} \Upsilon_i^{(j)}  - \Xi^{(b)}_i\right| \ge \lambda_\beta^{(b,\infty)}\\
\Rightarrow & \left|\langle \Phi_{\cdot i}^{(b,t)},  \bar \beta - \hat \beta^{(b,t)}  \rangle + \eta_b \sum_{j=1}^{b-1} \Upsilon_i^{(j)} \right| \ge \lambda_\beta^{(b,\infty)} - |\Xi^{(b)}_i| \ge \left(1- \frac2{(m+M)C_\beta} \right)\lambda_\beta^{(b,\infty)},
\end{aligned}
\end{equation}
leading that
\begin{equation}\label{eq: sa2}
\begin{aligned}
\sum_{i \in \mS^{(b,t+1)}\setminus \mathcal{\bar S} } \left( \Xi^{(b)}_i \right)^2 
\le&   \sum_{i \in \mS^{(b,t+1)}\setminus \mathcal{\bar S} } \left(  \frac2{m+M} \frac{\alpha_b}{ N_b} \right)^2 
= \sum_{i \in \mS^{(b,t+1)}\setminus \mathcal{\bar S} } \left(  \frac2{(m+M) C_\beta} \lambda_\beta^{(b, \infty)} \right)^2  \\
\le&\frac1{\left( \frac{(m+M)C_\beta}2-1  \right)^2} \sum_{i \in \mS^{(b,t+1)}\setminus \mathcal{\bar S} } \left|\langle \Phi_{\cdot i}^{(b,t)},  \bar \beta - \hat \beta^{(b,t)}  \rangle + \eta_b \sum_{j=1}^{b-1} \Upsilon_i^{(j)} \right|^2.
\end{aligned}
\end{equation}
Combining \eqref{eq: refine b b}, \eqref{eq: sa1}, and \eqref{eq: sa2}, we have
\begin{equation}\label{eq: up sharp} 
\begin{aligned}
&\left\| \hat \beta^{(b,t+1)}  - \bar \beta \right\|_2 \\
\le&  \frac{M}{m+M}\| \hat \beta^{(b,t)} - \bar \beta \|_2 + \eta_b\sum_{j=1}^{b-1}\left\| \Upsilon^{(j)}_{\mathcal{S}^{(b,t+1)} \cup \mathcal{\bar S}} \right\|_2  + \|\Xi_{\mathcal{\bar S}}^{(b)} \|_2\\
& + \frac1{ \frac{(m+M)C_\beta}2-1 } \sqrt{ \sum_{i \in \mS^{(b,t+1)}\cup\mathcal{\bar S} } \left|\langle \Phi_{\cdot i}^{(b,t)},  \bar \beta - \hat \beta^{(b,t)}  \rangle + \eta_b \sum_{j=1}^{b-1} \Upsilon_i^{(j)} \right|^2 } \\
\le& \frac{M}{m+M-(2/C_\beta)}\| \hat \beta^{(b,t)} - \bar \beta \|_2 + \frac{m+M}{m+M -2/C_\beta}\cdot \eta_b \sum_{j=1}^{b-1}\left\| \Upsilon^{(j)}_{\mathcal{S}^{(b,t+1)} \cup \mathcal{\bar S}} \right\|_2  + \|\Xi_{\mathcal{\bar S}}^{(b)} \|_2\\ 
\le& \frac{M}{m+M-(2/C_\beta)}\| \hat \beta^{(b,t)} - \bar \beta \|_2 
+ \frac{2}{m+M-(2/C_\beta)} \frac{C_e^2 C_\beta^2}{N_b}\sum_{j=1}^{b-1} \frac{sL_j \alpha_j^2 }{N_j^2}  +  \eta_b \theta_b \\
\le & \underbrace{\left( \frac{M}{m+M-(2/C_\beta)} \right)}_{=: \delta}\| \hat \beta^{(b,t)} - \bar \beta \|_2 
+\underbrace{ \left( \frac{2C_e}{m+M-(2/C_\beta)} + \frac2{m+M} \right)}_{=: C_r}  \frac{\theta_b}{N_b} ,
\end{aligned}
\end{equation}
where the third inequality follows from \eqref{eq: batch b rip} and \eqref{eq: upsilon}, and the last inequality follows from the definition of $b^*$ and $\theta_b$. 
By $C_\beta = \frac{8}{m+M} \cdot \frac{m+M+M/\kappa}{m+M - M/\kappa}$, we get $m >2 /C_\beta$ and consequently $ \delta\in(0,1)$. 
Therefore, after the $t_b$-th iteration, we perform additional $\Delta t$ further iterations with the fixed threshold $\lambda_\beta^{(b,\infty)}$, yielding 
\begin{equation}\label{eq: conv sharp}
    \left\| \hat \beta^{(b,t_b + \Delta t)}  - \bar \beta \right\|_2 \le \delta^{\Delta t} \left\| \hat \beta^{(b,t_b)}  - \bar \beta \right\|_2 + \frac{C_r}{1- \delta} \cdot  \frac{\theta_b}{N_b} .
\end{equation}
We then choose $\Delta t > \log_{1/\delta} \left(C_e \sqrt s N_b \lambda_\beta^{(b,\infty)} \right)$ and take $\hat \beta^{(b)} = \hat \beta^{(b, t_b + \Delta t)}$ as the output estimator for the renewable learning in the $b$-th batch.

In summary, for every $b \ge b^*$, the IHT-based algorithm demonstrates the refined error bound: 
$$
\left\| \hat \beta^{(b)}  - \bar \beta \right\|_2 \le \left(1+ \frac{C_r}{1- \delta} \right) \cdot \frac{\theta_b}{N_b}  .
$$

\section{Proof of GLM example}
Recall $N_b = \sum_{j=1}^b n_j$. Denote by $X_i \in \mathbb R^{1 \times p}$ the i-th observation of all $p$ covariates, and denote by $X_{i,S } \in \mathbb R^{1\times |S|}$ the i-th observation of the covariates indexed by set $S$.
Recall $f_j(\beta) = \sum_{i \in \mI_j}\Big\{g(X_i \beta) - Y_i X_i \beta \Big\}$, $\nabla f_j(\beta) = \sum_{i \in \mI_j} X_i^\top\Big\{g'(X_i \beta) - Y_i \Big\}$, and $\nabla^2 f_j(\beta) = \sum_{i \in \mI_j} X_i^\top X_i g''(X_i \beta) $.
Define $\xi_i = Y_i - g'(X_i \beta^*)$.
It remains to verify Assumptions \ref{assump: rip}, \ref{assump: lipschitz}, and \ref{assump: N1}, and to establish upper bounds on the gradient of the cumulative loss function at $\beta^*$. 

\subsection{Verify Assumption \ref{assump: rip}}
Define $C_s' : = \left( \frac{ 1+ 4K^2 U^2 (1+1/\kappa)}{1+ 4K^2 U^2 (1-1/\kappa)} \right)^2>1 $ is a universal constant depend solely on $K, U$, and $\kappa$, where we can choose the decay rate $\kappa \in \left( \frac{ 4K^2 U^2 }{1+ 4K^2 U^2  } ,~1\right)$. 
By Lemma \ref{le: rip and max} and Weyl's inequality, with probability at least $1 - 3\exp(-6 s \log p)$ we have that,
\begin{equation}\label{eq: weyl}
\sup_{j\in \mathbb N_+} ~\sup_{\mS\subset [p]: |\mS| \le (2C_s'+1) s}~ \max_{1\le k \le |\mS|}~ \left| \Lambda_k\left( \frac1{n_j} \sum_{i \in \mI_j} X_{i,\mS}^\top X_{i,\mS}\right) - \Lambda_k (\Sigma_{\mS, \mS}^{(j)})\right| \le \frac{1}{2K}.
\end{equation}
Following Assumption \ref{assump: g}, we get $g''(X_i \beta) \in [U^{-1}, U]$ for every $X_i \in \mathbb R^{1 \times p}$ and $\beta \in \mathbb B_0^p ( (1+2C_s')s)$, By applying \eqref{eq: weyl}, for every $j \in \mathbb N_+$ we get 
$$
\begin{aligned} 
\sup_{\beta \in \mathbb B_0^p ( (1+2C_s')s)}~ \sup_{\mS\subset [p]: |\mS| \le (2C_s'+1) s} \left\| \frac{1}{n_j} \sum_{i\in \mI_j } g''(X_i \beta ) X_{i, \mS}^\top X_{i, \mS}  \right\|_2
& \le U \left( \| \Sigma^{(j)}\|_2 + \frac1{2K} \right) \le 2KU,
\end{aligned}
$$
A similar technique leads to 
$$
\begin{aligned} 
\inf_{\beta \in \mathbb B_0^p ( (1+2C_s')s)}~ \inf_{\mS\subset [p]: |\mS| \le (2C_s'+1) s} \Lambda_{\min} \left( \frac{1}{n_j} \sum_{i\in \mI_j } g''(X_i \beta ) X_{i, \mS}^\top X_{i, \mS} \right)
& \ge U^{-1} \left( \Lambda_{\min} (\Sigma^{(j)}) - \frac1{2K} \right) \ge \frac{1}{2KU}.
\end{aligned}
$$
Therefore, with each sample size $n_j \ge 32 C_K (2C_s'+1)s \log p +  32 C_K \log j$ and under a probability greater than $1-3p^{-6s}$, we verify that Assumption \ref{assump: rip} holds in the GLM setting , i.e., each $f_j$ satisfies the condition RSS$\left( m=(2KU)^{-1}, ~M=2KU, ~(2C_s'+1)s\right)$.

\subsection{Verify Assumption \ref{assump: lipschitz}}
Define two universal constants $C_e' := 1+ \sqrt{C_s'} $ and $C_\beta' := \frac{8}{m+M} \sqrt{C_s'} = \frac{16KU}{1+4K^2 U^2} \sqrt{C_s'}$.
Under the GLM setting, we first establish the fact 
$$
\begin{aligned}
&\nabla f_j ( \beta^*) - \nabla f_j (\gamma) -  \nabla^2 f_j (\gamma)  ( \beta^* - \gamma)\\
=& \sum_{i \in \mI_j} X_i^\top \Big\{ g'(X_i \beta^*) - g'(X_i \gamma) - g''(X_i \gamma) \cdot X_i (\beta^* - \gamma) \Big\}\\
=& \sum_{i \in \mI_j} X_i^\top \Big\{ g''(X_i \tilde\beta_i)\cdot X_i(\beta^* -  \gamma) - g''(X_i \gamma) \cdot X_i (\beta^* - \gamma) \Big\},
\end{aligned}
$$
where, because \(g\) is twice differentiable (Assumption \ref{assump: g}), the mean value theorem guarantees for each \(i \in \mI_j\) the existence of
\(\tilde\beta_i = u_i\beta^*+(1-u_i)\gamma\) with \(u_i\in(0,1)\) (depending on \(X_i\beta^*\) and \(X_i\gamma\)) so that the last equality holds.
Then, for each $j \in \mathbb N_+$ and each $\gamma \in \mathbb B_0^p ((2C_s'+1)s)$ satisfying $\| \beta^* - \gamma\|_2 \le 4C_e' C_\beta' \sqrt{aKU} \cdot \sqrt{\frac{s \log(jp)}{N_j}}$, we have
$$
\begin{aligned}
&\sup_{\mathcal S \subset [p]: ~|\mathcal S| \le (2C_s'+1)s} \left\| \left\{ \nabla f_j ( \beta^*) - \nabla f_j (\gamma) -  \nabla^2 f_j (\gamma)  ( \beta^* - \gamma) \right\}_{\mathcal S} \right\|_2 \\
=& \sup_{\mathcal S \subset [p]: ~|\mathcal S| \le (2C_s'+1)s} \left\|  \sum_{i \in \mI_j} X_{i,\mathcal S}^\top \Big\{ g''(X_i \tilde\beta_i) - g''(X_i \gamma)  \Big\} \cdot X_i (\beta^* - \gamma) \right\|_2 \\
\le & \sup_{\mathcal S \subset [p]: ~|\mathcal S| \le (2C_s'+1)s}   \sum_{i \in \mI_j} \left\|X_{i,\mathcal S}^\top\right\|_2 \cdot \left| g''(X_i \tilde\beta_i) - g''(X_i \gamma)  \right| \cdot \left| X_i (\beta^* - \gamma) \right| \\
\overset{(i)}{\le} & \sup_{\mathcal S \subset [p]: ~|\mathcal S| \le (2C_s'+1)s} ~\max_{i' \in \mI_j} \left\|X_{i',\mathcal S}^\top\right\|_2 \cdot  \sum_{i \in \mI_j} C_g u_i  \left| X_i (\beta^* - \gamma) \right|^2\\
\overset{(ii)}<& C_g\sqrt{12 K (2C_s' +1)} \sqrt{s \log p + \log N_j}~  (\beta^* - \gamma)^\top  \left(\sum_{i \in \mI_j} X_i^\top X_i \right) (\beta^* - \gamma) \\
\le & 8 C_g K\sqrt{ K (2C_s' +1)} \cdot n_j\sqrt{s \log p + \log N_j} \cdot \left\|\beta^* - \gamma \right\|_2^2,
\end{aligned}
$$
where:
\begin{itemize}
\item Inequality (i) follows from the local Lipschitz property as introduced in Assumption \ref{assump: g} (with $C_\delta = 4C_e'C_\beta' \sqrt{aKU}$).
This inequality also follows from the relationships $X_i \tilde \beta_i -X_i \gamma = u_i X_i(\beta^* - \gamma)$ and $\tilde \beta_i - \beta^* = (u_i-1) (\beta^* - \gamma)$.

\item Inequality (ii) follows from inequality \eqref{eq: sparse Lj} in Lemma \ref{le: rip and max}, which holds for every $j \in \mathbb N_+$ with a probability greater than $1-2p^{-4s}$.

\item The last inequality follows from the relationship $\| \beta^* - \gamma\|_0 \le (2C_s' +1)s$, and from inequality \eqref{eq: sparse operator} in Lemma \ref{le: rip and max}, which requires each sample size satisfying $n_j \ge 32 C_K (2C_s'+1)s \log p +  32 C_K \log j$, and holds with a probability greater than $1-3p^{-6s}$.
\end{itemize}

Furthermore, by Lemma \ref{le: error}, we can take $\alpha_j = 4\sqrt{aUK } \cdot \sqrt{N_j\log(jp)}$ as an upper bound of $\| \sum_{k=1}^j \nabla f_k (\beta^*)\|_\infty$, which holds for every $j \in \mathbb N_+$ with a probability greater than $1-3p^{-2s}-3p^{-3}$. 
Then, by taking $\lambda_\beta^{(j,\infty)} = C_\beta' \frac{\alpha_j}{N_j} = 4C_\beta' \sqrt{aUK } \cdot \sqrt{\frac{\log(jp)}{N_j}}$, we get 
$$
C_e' \sqrt s \lambda_\beta^{(j,\infty)} = C_e' C_\beta' \frac{\sqrt s \alpha_j}{N_j} =4C_\beta' C_e' \sqrt{aKU } \cdot \sqrt{\frac{s\log(jp)}{N_j}}. 
$$
Therefore, with a probability greater than $1-O(p^{-2})$, we learn that each $f_j$ satisfies the condition RGS$\left(L_j,~\delta_j, ~  (C_s'+1) s \right)$ with 
\begin{equation}\label{eq: L delta GLM}
\begin{aligned}
L_j &= 8 C_g K\sqrt{ K (2C_s' +1)} \cdot n_j\sqrt{s \log p + \log N_j},\\
\quad \delta_j &= C_e' \sqrt s \lambda_\beta^{(j,\infty)}=4C_\beta' C_e' \sqrt{aKU } \cdot \sqrt{\frac{s\log(jp)}{N_j}}.
\end{aligned}
\end{equation}

\subsection{Verify Assumption \ref{assump: N1}}
Finally, we verify Assumption \ref{assump: N1} with a proper constant $C_p'$.
For every $b\ge 1$, we establish that
\begin{equation}\label{eq: s L alpha}
\begin{aligned}
\sum_{j=1}^b \frac{sL_j}{N_j^2} \cdot \alpha_j^2
=&\underbrace{ 128a C_g  U K^2 \sqrt{ K (2C_s' +1)} }_{=: ~C_\alpha}  \cdot \sum_{j=1}^b \frac{s n_j \sqrt{s\log p + \log N_j} \cdot ( \log p + \log j)}{N_j}\\
\le& C_\alpha \sum_{j=1}^b \frac{ n_j (s\log p + s \log N_j )\cdot \sqrt{ s\log p + s\log j} }{N_j}\\
\le & C_\alpha \sqrt{ s\log p + s\log b} \cdot (s\log p + s \log N_b) \cdot \sum_{j=1}^b \frac{n_j}{N_j} \\
\le &   C_\alpha \sqrt{ s\log p + s\log b} \cdot (s\log p + s \log N_b) \cdot \left( 1+ \log (N_b/N_1) \right),
\end{aligned}
\end{equation}
where the last inequality follows from $x/y< \log \frac{y}{y-x}$ for every $0<x<y$.
Therefore, to verify Assumption \ref{assump: N1}, we only need to show that 
$$
\frac{C_p C_\alpha}{4\sqrt{aUK }}  \cdot (s\log p + s \log N_b) \cdot \log N_b \le \sqrt{N_{b} }
$$
holds for every $b \ge 1$, where $C_p =  \frac{4(C_e')^2 C_\beta'}{m+M} =  \frac{8KU(C_e')^2 C_\beta'}{1+ 4K^2 U^2}$ as required by Proposition 1 and Theorem 1.

Define the function $h(x):= \frac{\sqrt x}{\log x}$ and $h$ is monotonically increasing in the case $x \ge e^2$. Therefore, when the initial sample size satisfies $\sqrt N_1 \ge\max\left( \frac{C_p C_\alpha }{2 \sqrt{aKU}} s (\log p) (\log N_1),~e  \right)$, we conclude that
$$
\frac12 \cdot \frac{\sqrt{N_b}}{\log N_b} \ge \frac12 \cdot \frac{\sqrt{N_1}}{\log N_1} \ge \frac{C_p C_\alpha }{4 \sqrt{aKU}}  s \log p.
$$
Similarly, the function $x \mapsto \frac{\sqrt{x}}{(\log x)^2}$ is monotonically increasing in the case $x \ge e^4$, and thus under the setting $\sqrt N_1 \ge\max\left( \frac{C_p C_\alpha }{2 \sqrt{aKU}} s (\log N_1)^2,~e^2  \right)$, we have
$$
\frac12 \cdot \frac{\sqrt{N_b}}{(\log N_b)^2} \ge \frac12 \cdot \frac{\sqrt{N_1}}{( \log N_1)^2} \ge \frac{C_p C_\alpha }{4 \sqrt{aKU}}  s .
$$
Therefore, with the initial sample size satisfying
$$
n_1 \ge \max\left( \frac{(C_p C_\alpha)^2 }{4aKU} s^2 (\log n_1)^4,~ \frac{(C_p C_\alpha)^2 }{4aKU} s^2 (\log p)^2 (\log n_1)^2,~ e^4  \right)
\asymp s^2 (\log^2n_1) \log^2(p+n_1),
$$
we establish that Assumption \ref{assump: N1} holds.

\subsection{Proof of Theorem \ref{th: glm}}\label{sec: proof t3}
Based on the verification above, we establish that under Assumptions \ref{assump: cov} and \ref{assump: g} (with constants \(U\), \(2C_s' + 1\), \(C_\delta = 4C_e' C_\beta' \sqrt{aKU}\), and \(C_g\)), with a probability greater than \(1 - 11p^{-2}\), Assumptions \ref{assump: rip}-\ref{assump: N1} hold simultaneously for every batch \(j \in \mathbb{N}_+\).
Specifically, the corresponding parameters in these assumptions are given by \(M = 2KU\), \(m = (2KU)^{-1}\), \(C_p = \frac{8KU (C_e')^2 C_\beta'}{1 + 4K^2 U^2}\), and \(\alpha_j = 4\sqrt{aUK} \cdot \sqrt{N_j \log(jp)}\). The smoothness parameters \(L_j\) and \(\delta_j\) are defined as in \eqref{eq: L delta GLM}.

Then, suppose the batch sample sizes satisfy $n_1 \ge \max\left( \frac{(C_p C_\alpha)^2 }{4aKU} s^2 (\log^2 n_1) \log^2(p+n_1), ~ e^4  \right) $ for the initial batch and $n_j \ge 32 C_K (2C_s' + 1) (s \log p + \log j)$ for all $j \ge 2$.
Under these conditions, by setting the learning rate $\eta_j \in \left[ \frac{2KU}{(1+4K^2 U^2) N_j},~\frac{4KU}{(1+4K^2 U^2) N_j}\right]$, the decay rate $\kappa \in \left( \frac{ 4K^2 U^2 }{1+ 4K^2 U^2  } ,~1\right)$, and the regularization thresholds such that
\begin{equation}\label{eq: glm lambda}
     \lambda_\beta^{(j,0)} \ge  \frac{1}{2\sqrt s}\|  \beta^* \|_2 ,
    \quad \lambda_\beta^{(j,\infty)}  = 4C_\beta' \sqrt{aKU } \cdot \sqrt{\frac{\log(jp)}{N_j}},
\end{equation}
following the results in Theorem 1, we establish the following joint estimation bounds for all $j \in \mathbb{N}_+$:
$$
\| \hat \beta^{(j)} - \beta^* \|_2 \le 4\sqrt{aUK}C_e' C_\beta' \cdot \sqrt{\frac{s(\log p + \log j)}{N_j}} ,
\quad \| \hat \beta^{(j)} \|_0 \le  (1+C_s')\cdot s,
$$
with probability exceeding $1-11p^{-2}$.
Therefore, we complete the proof of Theorem \ref{th: glm}.

\subsection{Proof of Theorem \ref{th: glm sharper}}
By Theorem \ref{th: sharper batch}, it suffices to clarify the $b^*$, specifically defined through \eqref{eq: b*}.
In the GLM setting, by Lemma \ref{le: error}, we can take $\alpha_b = 4\sqrt{aKU} \sqrt{N_b \log(bp)}$ and $\theta_b = \sqrt{6aKU} \sqrt{N_b (s+ \log(2b^2/ \varrho ) )}$, which leads that 
$$
\begin{aligned}
\left\|\sum_{k=1}^b \nabla f_k(\beta^*) \right\|_\infty  &= \left\| \sum_{ i \in [N_b]} X_{i}^\top \xi_i \right\|_\infty \le \alpha_b,\\
\left\|\sum_{k=1}^b \left\{ \nabla f_k(\beta^*) \right\}_{\mathcal S^*} \right\|_2  &= \left\| \sum_{ i \in [N_b]} X_{i, \mathcal S^*}^\top \xi_i \right\|_2 \le \theta_b,\\
\end{aligned}
$$
simultaneously hold for every $b \ge 1$ with a probability greater than $1-\varrho- 6p^{-2}$.
Then, following the parameters introduced in Section \ref{sec: proof t3}, we first rewrite $b_1^*$ as 
\begin{equation}\label{eq: b1 glm}
    b_{1,GLM}^*:=  \inf\left\{ b \ge 1: ~ \sqrt{N_b} \ge \frac{16\sqrt a (KU)^{3/2} \left( 1 + 16 C_s' \right)}{1+4K^2 U^2} \cdot \frac{ \sqrt{\log p + \log b}}{\min_{i \in\mS^*} | \beta_i^* | } \right\}.
\end{equation} 

By the inequality \eqref{eq: s L alpha}, we have
$$
\sum_{j=1}^b \frac{sL_j}{N_j^2} \cdot \alpha_j^2
\le \sqrt2  C_\alpha \left\{ (s\log p)^{3/2} \cdot \log N_b + s^{3/2} \cdot \log^{5/2} N_b  \right\},
$$
therefore, it suffices to find an batch index ${b_{2,GLM}^*}$ such that
\begin{equation}\label{eq: b2 glm}
\sqrt2 C_e' (C_\beta')^2   C_\alpha \left\{ (s\log p)^{3/2} \cdot \log N_{b} + s^{3/2} \cdot \log^{5/2} N_{b}  \right\}
\le  \sqrt{6aKU} \sqrt{s N_b }, \quad\text{for every } b \ge {b_{2,GLM}^*}.
\end{equation}

Since $h(x) = \frac{\sqrt x}{\log x}$ is monotonically increasing in the case $x \ge e^2$, there must exist an index $b'\ge 1$ such that 
$$
\frac{\sqrt{N_b}}{\log N_b} \ge \frac{2\sqrt2 C_e' (C_\beta')^2   C_\alpha}{\sqrt{6aKU}} s  \log^{3/2} p, \quad\text{for every } b \ge b'.
$$
Similarly, there must exist an index $b''\ge 1$ such that 
$$
\frac{\sqrt{N_b}}{\log^{5/2} N_b} \ge \frac{2\sqrt2 C_e' (C_\beta')^2   C_\alpha}{\sqrt{6aKU}} s , \quad\text{for every } b \ge b''.
$$
Taking $b_{2,GLM}^* = b' \vee b''$ leads that \eqref{eq: b2 glm} holds. 

Consequently, for any batch index $b$ satisfying the sample size condition
\begin{equation}\label{eq: sharp glm sample}
 N_b \ge C_{4} \max\left\{ s^2 (\log^2  N_b) \cdot  \log^3 (p \vee N_b)  ,~ \frac{ \log p + \log b}{\min_{i \in\mS^*} | \beta_i^* |^2 } \right\},
\end{equation}
the following refined $\ell_2$ error bound consistently holds with probability at least $1 - \varrho - 11p^{-2}$:
\begin{equation}\label{eq: sharp glm}
\| \hat \beta^{(b)} - \beta^*\|_2 \le C_{sharp}' \cdot \sqrt{\frac{s +\log (2b^2 / \varrho)}{N_b}} ,
\end{equation}
where $C_4$ and $C_{sharp}'$ are universal constants depending solely on the parameters $a, K, U, C_g,$ and $\kappa$.

\paragraph{Almost full recovery}
We next consider the support recovery performance under the setting $s \succ 1$.
From Theorem \ref{th: glm} we learn that the estimator $\hat \beta^{(b)}$ is $(1+C_s') s$-sparse. 
For any batch \(b\) satisfying \eqref{eq: sharp glm sample}, within this batch, suppose we first perform \(t_b\) iterations to decay the threshold parameter from \(\lambda_\beta^{(b,0)}\) to \(\lambda_\beta^{(b,\infty)}\). Then, for any \(t \ge t_b\), following the proof of Theorem \ref{th: sharper batch}, we obtain
$$
\begin{aligned}
&|\mathcal S^* \setminus \mathcal S^{(b,t+1)}| + | \mathcal S^{(b,t+1)} \setminus \mathcal S^* | \\
\le & \sum_{i \in \mS^*} \mathbf{1} \left( |H^{(b,t+1)}_i| <\lambda_\beta^{(b,\infty)} \right)  + \sum_{i \in \mathcal S^{(b,t+1)} \setminus \mathcal S^*} \mathbf{1} \left( |H^{(b,t+1)}_i| \ge \lambda_\beta^{(b,\infty)} \right) \\
\le & \frac1{\left( \frac{(m+M)C_\beta'}2-1\right)^2} \cdot \frac1
{\left(\lambda_\beta^{(b,\infty)}\right)^2 } \sum_{i \in \mS^*}  \left|\langle \Phi_{\cdot i}^{(b,t)}, \beta^* - \hat \beta^{(b,t)}  \rangle + \eta_b \sum_{j=1}^{b-1} \Upsilon_i^{(j)} \right|^2 \\ 
&~+\frac1{\left( 1-\frac2{(m+M)C_\beta'}\right)^2} \cdot \frac1
{\left(\lambda_\beta^{(b,\infty)}\right)^2 } \sum_{i \in \mathcal S^{(b,t+1)} \setminus \mathcal S^*}  \left|\langle \Phi_{\cdot i}^{(b,t)}, \beta^* - \hat \beta^{(b,t)}  \rangle + \eta_b \sum_{j=1}^{b-1} \Upsilon_i^{(j)} \right|^2 \\
\le&\frac2{\left( 1-\frac2{(m+M)C_\beta'}\right)^2 \cdot\left(\lambda_\beta^{(b,\infty)}\right)^2} \cdot
\left\{ \sum_{i \in \mathcal S^{(b,t+1)} \bigcup\mathcal S^* } \left\langle \Phi_{\cdot i}^{(b,t)}, \beta^* - \hat \beta^{(b,t)} \right \rangle^2 
+ \left( \eta_b \sum_{j=1}^{b-1} \left\| \Upsilon_{\mathcal S^{(b,t+1)} \bigcup\mathcal S^*}^{(j)} \right\|_2 \right)^2 \right\}
\end{aligned}
$$
where the second inequality follows from inequalities \eqref{eq: sa1} and \eqref{eq: sa2-1}, with taking $m = (2KU)^{-1}$ and $M = 2KU$.
Therefore, by following \eqref{eq: up sharp} and \eqref{eq: glm lambda}, with a constant $C_{rec}$ (solely depending on $K, U, a, \kappa$, and $C_g$), under a probability greater than $1- 11p^{-2}$,
$$
|\mathcal S^* \setminus \mathcal S^{(b,t+1)}| + | \mathcal S^{(b,t+1)} \setminus \mathcal S^* | 
\le  \frac{C_{rec}N_b}{\log p + \log b} \cdot \left\{ \left\| \beta^* - \hat \beta^{(b,t)} \right \|_2^2 
+ \left( \frac{\theta_b}{N_b}\right)^2 \right\}.
$$
Furthermore, by \eqref{eq: conv sharp}, with sufficiently large $t \ge t_b + C \log N_b$, we can get the refined estimation rate as introduced in \eqref{eq: sharp glm}.
We take $\varrho = p^{-2}$ and $\theta_b = \sqrt{6aKU} \sqrt{N_b (s+ \log(2b^2 p^2  ) )}$. Then by following \eqref{eq: sharp glm sample} and \eqref{eq: sharp glm}, with a probability greater than $1- 12p^{-2}$ we obtain
$$
\begin{aligned}
|\mathcal S^* \setminus \mathcal S^{(b,t+1)}| + | \mathcal S^{(b,t+1)} \setminus \mathcal S^* | 
\le&  \frac{C_{rec}N_b}{\log p + \log b} \cdot \left( C_{sharp}'^2 + 6aKU\right) \frac{s + \log(2b^2 p^2)}{N_b}  \\
\asymp & \frac{s + \log b +\log p}{\log p + \log b} \\
= & \frac{s }{\log p + \log b} + 1\\
\prec & s,
\end{aligned}
$$
where the last inequality follows from $s \succ 1$.
Therefore, we prove the almost full support recovery, which completes the proof of Theorem \ref{th: glm sharper}.

\section{Technical lemmas}
Recall we denote by $X_i \in \mathbb R^{1 \times p}$ the i-th observation of all $p$ covariates, and denote by $X_{i,S } \in \mathbb R^{1\times |S|}$ the i-th observation of the covariates indexed by set $S$. 
\begin{lemma}[Designs]\label{le: rip and max}
Suppose Assumption \ref{assump: cov} holds. For arbitrary two constants $C_a, C_b \ge 1$, with the sample size in each batch satisfying $n_j \ge 32 C_K ( C_a s \log p + \log j )$ (where $C_K>0$ is a constant depending solely on $K$), we have
    \begin{equation}\label{eq: sparse operator}
    \mathbf P  \left\{ \sup_{j \in \mathbb N_+} \sup_{S \subset[p]: |S| \le C_{a} s }  \left\| \frac1{n_j} \sum_{i \in \mI_j} X_{i,S}^\top X_{i,S} - \Sigma_{SS}^{(j)} \right\|_2 >\frac{1}{2K} \right\} \le 3 e^{-2 C_a s \log p } ,
\end{equation} 
and
\begin{equation}\label{eq: sparse Lj}
\mathbf P  \left\{ \sup_{j \in \mathbb N_+}~ \sup_{S \subset[p]: |S| \le C_b s} ~\max_{i \in \cup_{k\in[j]} \mI_k}\left(\left\| X_{i,S} \right\|_2 -  \sqrt{12K}  \sqrt{ C_b s\log  p  + \log N_j } \right)>0 \right\} \le 2e^{-2C_b s \log  p } .
\end{equation}
\end{lemma}
\begin{proof}[Proof of Lemma \ref{le: rip and max}]
For a fixed batch index $j$ and a fixed set $S$ satisfying $S\subset [p]$ and $|S| = C_a s$, by Remark 5.40 in \cite{vershynin2010introduction}, we obtain
\begin{equation}\label{eq: spec norm}
\mathbf P \left( \left\| \frac1{n_j} \sum_{i \in \mI_j} X_{i,S}^\top X_{i,S} - \Sigma_{SS}^{(j)} \right\|_2 > \max(\iota_j , \iota_j^2)  \right) \le 2e^{- u_j},
\end{equation}
where $\iota_j = \left(C \vee \frac1{\sqrt c}\right) \cdot \left(\sqrt{\frac{C_1 s}{n_j} } +\sqrt{\frac{u_j}{n_j} } \right)$ and $C,c$ are two constants depending only on $ \| \Sigma^{(j)}\|_2$.
By taking $u_j =3C_a s\log p + 4\log j $ and assuming $n_j \ge 32K^2 \left(C^2 \vee \frac1{ c}\right)( C_a s \log p + \log j )$, we get $\iota_j \vee \iota_j^2 = \iota_j \le 1/(2K) $, which yields that
\begin{equation}\label{eq: X Sigma}
\begin{aligned}
&\mathbf P \left( \sup_{j\in \mathbb N_+}\sup_{S\subset [p]:~ |S| \le C_a s} \left\| \frac1{n_j} \sum_{i \in \mI_j} X_{i,S}^\top X_{i,S} - \Sigma_{SS}^{(j)} \right\|_2 > \frac1{2K} \right)\\
\le&  \sum_{j \ge 1}  \sum_{ S\subset [p]:~ |S| = C_a s} \mathbf P \left( \left\| \frac1{n_j} \sum_{i \in \mI_j} X_{i,S}^\top X_{i,S} - \Sigma_{SS}^{(j)} \right\|_2 > \max(\iota_j , \iota_j^2)  \right)\\
\le& \sum_{j \ge 1} 2\binom p{ C_a s} \exp\left(- 3C_a s\log p - 4 \log j\right)\\
\le&  2 \exp\left( -2C_a s\log p \right) \sum_{j \ge 1}\frac{1}{j^4}\\
\le & 3 \exp\left( -2C_a s\log p \right),
\end{aligned}
\end{equation}
where the last inequality follows from $\sum_{j\ge 1} j^{-4} < 3/2$. Hence, we prove \eqref{eq: sparse operator}.

Similarly, for a fixed $j \in \mathbb N_+$ and an observation index $i \in [N_j] = \bigcup_{k=1}^j \mathcal{I}_k$, let $b(i)$ denote the batch membership of observation $i$, such that $b(i) = k$ if $i \in \mathcal{I}_k$.
Then for a fixed set $S$ satisfying $S\subset [p]$ and $|S| = C_b s$, by Assumption 4 and Theorem 2.1 in \citet{hsu2012tail}, we obtain
\begin{equation} 
\mathbf P \left( \left\| X_{i,S}\right\|_2^2 > tr(\Sigma_{SS}^{b(i)}) +  2 \sqrt{tr \left(\Sigma_{SS}^{b(i)} \Sigma_{SS}^{b(i)} \right) u_i } + 2 \| \Sigma_{SS}^{b(i)} \|_2 u_i  \right) \le e^{- u_i}.
\end{equation}
Since $ \sup_{j \ge 1} \| \Sigma^{(j)}\|_2 \le K$, we further have
$$
tr(\Sigma_{SS}^{b(i)}) +  2 \sqrt{tr \left(\Sigma_{SS}^{b(i)} \Sigma_{SS}^{b(i)} \right) u_i } + 2 \| \Sigma_{SS}^{b(i)} \|_2 u_i
\le2 C_b Ks + 3K u_i.
$$
Then by taking $u_i = 3C_b s \log p + 4 \log N_j$, we further have
$$
\begin{aligned}
& \mathbf P  \left\{ \sup_{j \in \mathbb N_+} \sup_{S \subset[p]: |S| \le C_b s} \max_{i \in [N_j]} \left(\left\| X_{i,S} \right\|_2 - \sqrt{12K}  \sqrt{ C_b s\log  p  + \log N_j } \right)>0 \right\} \\
\le&  \sum_{j \ge 1} ~ \sum_{ S\subset [p]:~ |S| = C_b s} ~\sum_{i \in [N_j] } \mathbf P \left( \left\| X_{i,S} \right\|_2^2 \ge 12 K C_b s \log p + 12 K \log N_j \right)\\
\le& \sum_{j \ge 1} \binom p{ C_b s} N_j  \exp\left(- 3C_b s\log p - 4 \log N_j\right)\\
\le& \exp\left( -2C_b s\log p \right) \sum_{j \ge 1}\frac{1}{N_j^3}\\
\le& 2\exp\left( -2C_b s\log p \right),
\end{aligned}
$$
where the last inequality follows from $\sum_{j\ge 1} N_j^{-3} < \sum_{j\ge 1} j^{-3} < 2$. Hence, we prove \eqref{eq: sparse Lj} and complete the proof of Lemma \ref{le: rip and max}.
\end{proof}

\begin{lemma}[Sub-Gaussian errors]\label{lemma: subgaussian}
Suppose that the GLM setting holds and define $\xi_i : =Y_i - g'(X_i \beta^*)$. Then based on Assumption \ref{assump: g}, each $\xi_i$ is sub-Gaussian with zero mean and sub-Gaussian parameter $\sqrt{aU}$, that is,
\begin{equation*}
    \mathbf{P}\left( |Y_i- b'(\zeta_i^*)| \ge t\right)
    \le 2 \exp\left( -\frac{t^2}{2aU} \right), ~\text{for all}\ t \ge 0,~ i \ge 1.
\end{equation*}
\end{lemma} 
\begin{proof}[Proof of Lemma \ref{lemma: subgaussian}]
Under the GLM setting, we have $\mathbf E_{Y_i | X_i }(Y_i) = g'(X_i \beta^*)$. From Theorem 5.10 of \citet{lehmann2006theory}, we also have
\begin{equation*}
    \mathbf E_{Y_i | X_i }\left( \exp(\lambda Y_i) \right) = \exp\left( \frac{ g(X_i \beta^* + \lambda a )- g(X_i \beta^*) }{a}\right),~ \forall \lambda \in \mathbb R,
\end{equation*}
which leads to
\begin{equation*}
\begin{aligned}
\mathbf E_{Y_i,~ X_i } \left( e^{\lambda (Y_i- g'(X_i \beta^*))} \right) 
&= \mathbf E_{ X_i } \mathbf E_{Y_i | X_i } \left( e^{\lambda (Y_i- g'(X_i \beta^*))} \right) \\
&=\mathbf E_{ X_i } \exp\left( \frac{ g( X_i \beta^* + \lambda a )- g(X_i \beta^*) -\lambda a\cdot g'( X_i \beta^*)}{a}\right)\\
&\overset{(i)}{=} \mathbf E_{ X_i } \exp\left( \frac{ \lambda^2 a g''( \zeta_i)}{2 }\right)\\
&\le \exp\left( \frac{ \lambda^2 a U}{2 }\right), ~\forall \lambda \in \mathbb R,
\end{aligned}
\end{equation*}
where in equality (i), $ \zeta_i$ is between $X_i \beta^*$ and $X_i \beta^* + \lambda a$ based on Taylor's Theorem, and the last inequality follows from Assumption \ref{assump: g}.
Hence, we prove that $Y_i - g'( X_i \beta^* )$ is sub-Gaussian with zero mean and sub-Gaussian parameter $ \sqrt{aU}$, which completes the proof of Lemma \ref{lemma: subgaussian}. 
\end{proof}

\begin{lemma}[Stochastic errors]\label{le: error}
 Suppose Assumption \ref{assump: g} holds. Under all conditions of Lemma \ref{le: rip and max}, we have
 \begin{equation}\label{eq: concentrate grad}
\mathbf P  \left\{ \sup_{j \in \mathbb N_+}  \left(\left\| \sum_{ i \in [N_j]} X_{i}^\top \xi_i \right\|_\infty- 4 \sqrt{aU} \sqrt{ K N_j \log(jp)} \right)>0 \right\} \le 3p^{-2s} + 3p ^{-3},
\end{equation}
and 
\begin{equation}\label{eq: XiS*}
\mathbf P \left\{ \sup_{j \in \mathbb N_+}  \left( \left\| \sum_{ i \in [N_j]} X_{i, \mS^*}^\top \xi_i \right\|_2  - \sqrt{6aKU}  \sqrt{ N_j \left(s + \log(2j^2/\varrho) \right)} \right) >0 \right\} \le 3p^{-2s}+\varrho,
\end{equation}
where recall $\mathcal S^* = \text{supp}(\beta^*)$.
\end{lemma}

\begin{proof}[Proof of Lemma \ref{le: error}]
Define $X^{(j)} \in \mathbb R^{N_j \times p}$ the cumulated design matrix up to the $j$-th batch, and thus $X^{(j)} = \left( X_1^\top , \cdots, X_{N_j}^\top \right)^\top$. 
Also define $\xi^{(j)} = (\xi_1, \cdots, \xi_{N_j})^\top \in \mathbb R^{N_j \times 1 }$ the cumulated sub-Gaussian vector (proved by Lemma \ref{lemma: subgaussian}) up to the $j$-th batch.
Therefore $\sum_{ i \in [N_j]} X_{i}^\top \xi_i  = \left( X^{(j)} \right)^\top \xi^{(j)} \in \mathbb R^{p \times 1}$.

\textbf{On inequality \eqref{eq: concentrate grad}.}
For a given batch index $j$ and a given design $X^{(j)}$, it is straightforward that $ \left( X^{(j)}_{\cdot, k} \right)^\top \xi^{(j)}$ is a sub-Gaussian random variable with sub-Gaussian parameter $ \sigma = \sqrt{aU} \| X^{(j)}_{\cdot, k} \|_2$ (see Chapter 2 in \citep{wainwright2019high}), where $X^{(j)}_{\cdot, k} \in \mathbb R^{N_j \times 1}$ is the $k$-th column of $X^{(j)}$.
Define the event 
$$
\mathcal E_X = \left\{ \sup_{j \in \mathbb N_+} \sup_{S \subset[p]: |S| \le s}  \left\| \frac1{n_j} \sum_{i \in \mI_j} X_{i,S}^\top X_{i,S} - \Sigma_{SS}^{(j)} \right\|_2 \le \frac{1}{2K} \right\},
$$
and by Lemma \ref{le: rip and max} we have $\mathbf P(\mathcal E_X) \ge 1-3e^{-2s \log p}$.
And under $\mathcal E_X$ we derive that $ \sum_{i \in \mI_\ell} X_{i,k}^2 \le 2Kn_\ell$ holds for every $\ell \ge 1, k \in [p]$, leading that $\| X^{(j)}_{\cdot, k} \|_2^2 = \sum_{i \in[N_j]} X_{i,k}^2 \le 2K N_j$.
Therefore,
\begin{align*}
&\mathbf P  \left\{ \sup_{j \in \mathbb N_+} \max_{k \in [p]} \left(\left|\left( X^{(j)}_{\cdot, k} \right)^\top \xi^{(j)} \right| - 4 \sqrt{aU} \sqrt{ K N_j \log(jp)}  \right)>0 \right\}\\
\le& \mathbf P (\mathcal E_{X}^c) + \mathbf E_X \left\{ \mathbf1(\mathcal E_{X}) \cdot \sum_{j \ge 1} \sum_{k \in [p]}\mathbf P_{\xi^{(j)}|X } \left(  \left|\left( X^{(j)}_{\cdot, k} \right)^\top \xi^{(j)} \right|  > 4 \sqrt{aU} \sqrt{ K N_j \log(jp)}   ~\Big|~ X\right) \right\}\\ 
\le& 3p^{-2s} + p \sum_{j \ge 1} 2e^{-4 \log(j p)}\\
~\le& 3p^{-2s} + 3p ^{-3},
\end{align*}
where the first inequality applies the union bound, and the last inequality follows from $\sum_{j\ge 1} j^{-4} < 3/2$.

\textbf{On inequality \eqref{eq: XiS*}.}
By using Theorem 2.1 of \citet{hsu2012tail} and Lemma \ref{lemma: subgaussian}, for a given $j$ and given $X^{(j)}$, we have
\begin{equation*} 
  \mathbf P \left\{ \frac{\left\| \left( X^{(j)}_{\cdot, \mS^*} \right)^\top \xi^{(j)} \right\|_2^2}{aUN_j} \ge tr\left(\hat \Sigma_{\mS^*, \mS^*}^{(j)}\right)+ 2\left\|\hat \Sigma_{\mS^*, \mS^*}^{(j)} \right\|_F\sqrt{u_j} + 2\Lambda_{\max}\left( \hat \Sigma_{\mS^*, \mS^*}^{(j)} \right) u_j ~ \Bigg | j, X^{(j)}\right\} \le e^{-u_j},
\end{equation*}
where we define $\hat \Sigma_{\mS^*, \mS^*}^{(j)} : = \frac1{N_j} \sum_{i \in [N_j]} X_{i, \mS^*}^\top X_{i, \mS^*}$.
Combining with $|\mS^*|= s$, under event $\mathcal E_{X}$ we have
\begin{equation*}
\begin{aligned}
tr\left( \hat \Sigma_{\mS^*, \mS^*}^{(j)}\right) =&
 \sum_{k=1}^{s} \Lambda_k (\hat \Sigma_{\mS^*, \mS^*}^{(j)})
\le 2K s,\\
\left\|\hat \Sigma_{\mS^*, \mS^*}^{(j)} \right\|_F^2 = &
\sum_{k=1}^{s} \Lambda_k^2 \left( \hat \Sigma_{\mS^*, \mS^*}^{(j)} \right)
\le 4K^2s.
\end{aligned}
\end{equation*} 
Therefore, by taking $u_j = \log(2/\varrho) + 2 \log j$ (where $\varrho \in (0,1)$), we obtain
\begin{equation*}
\begin{aligned}
&  \mathbf P \left\{ \sup_{j \in \mathbb N_+}  \left( \left\| \sum_{ i \in [N_j]} X_{i, \mS^*}^\top \xi_i \right\|_2  - \sqrt{6aKU}  \sqrt{ N_j \left(s + \log(2j^2/\varrho) \right)} \right) >0 \right\} \\
\le&  \mathbf P (\mathcal E_{X}^c) + \mathbf E_X \left\{ \mathbf1(\mathcal E_{X}) \cdot \sum_{j \ge 1} \mathbf P \left( \left\| \sum_{ i \in [N_j]} X_{i, \mS^*}^\top \xi_i \right\|_2^2 \ge 6aKU\cdot  N_j\left( s + \log(2j^2/\varrho) \right)  ~\Big|~ X\right) \right\}\\
\le & 3p^{-2s} + \sum_{j \ge 1} \frac{\varrho}{2j^2}\\
\le & 3p^{-2s} +  \varrho,
\end{aligned}
\end{equation*}
where the last inequality follows from $\sum_{j\ge 1} j^{-2} < 2$. 
Hence, we complete the proof of \eqref{eq: XiS*} and the proof of Lemma \ref{le: error}.
\end{proof}

\bibliographystyle{unsrtnat}
\bibliography{sample}

\end{document}